\documentclass[twoside,11pt]{article}

\usepackage[preprint]{jmlr2e}
\usepackage{jmlr2e}

\usepackage{amsmath}
\usepackage{bm}
\usepackage{mathrsfs}
\usepackage[capitalize]{cleveref}

\crefformat{equation}{#2(#1)#3}

\DeclareFontFamily{U}{mathx}{\hyphenchar\font45}
\DeclareFontShape{U}{mathx}{m}{n}{
      <5> <6> <7> <8> <9> <10> gen * mathx
      <10.95> mathx10 <12> <14.4> <17.28> <20.74> <24.88> mathx12
      }{}
\DeclareSymbolFont{mathx}{U}{mathx}{m}{n}
\DeclareMathSymbol{\bigtimes}{1}{mathx}{"A1}

\NewDocumentCommand{\deffun}{}{\mathbin{:}}
\newcommand{\defeq}{\mathbin{:=}}

\NewDocumentCommand{\ld}{}{\operatorname{ld}}
\NewDocumentCommand{\card}{m}{\ensuremath{\operatorname{card}\left(#1\right)}}
\NewDocumentCommand{\supp}{m}{\operatorname{supp}\{#1\}}
\NewDocumentCommand{\closure}{m}{\ensuremath{\overline{#1}}}
\NewDocumentCommand{\dee}{}{\mathop{\mathrm{d}\!}}

\NewDocumentCommand{\rr}{o}{\ensuremath{\mathbb{R}\IfValueT{#1}{^{#1}}}}
\NewDocumentCommand{\nn}{o}{\ensuremath{\mathbb{N}\IfValueT{#1}{^{#1}}}}

\newcommand{\fSpace}[1]{\ensuremath{\mathcal{#1}}}
\newcommand{\set}[1]{\ensuremath{\mathcal{#1}}}
\newcommand{\mt}[1]{\ensuremath{\bm{#1}}}
\newcommand{\vc}[1]{\ensuremath{\bm{#1}}}

\NewDocumentCommand{\codecClass}{m m o}{\mathscr{C}_{#1,#2}\IfValueT{#3}{^{#3}}}
\NewDocumentCommand{\jointClass}{o}{\mathscr{C}_{\mathrm{joint}}\IfValueT{#1}{^{#1}}}
\NewDocumentCommand{\sepClass}{o}{\mathscr{C}_{\mathrm{sep}}\IfValueT{#1}{^{#1}}}
\NewDocumentCommand{\codec}{d<> o}{\mathcal{C}\IfValueT{#1}{_{#1}}\IfValueT{#2}{^{#2}}}
\NewDocumentCommand{\encoder}{d<> o}{\ensuremath{\delta_{E\IfValueT{#1}{,#1}}\IfValueT{#2}{^{#2}}}}
\NewDocumentCommand{\decoder}{d<> o}{\ensuremath{\delta_{D\IfValueT{#1}{,#1}}\IfValueT{#2}{^{#2}}}}
\NewDocumentCommand{\distortion}{m m o}{\ensuremath{\mathfrak{D}\IfValueT{#3}{_{#3}}\left(#1,#2\right)}}

\NewDocumentCommand{\dyadicU}{m m}{\set{U}_{#2}^{(#1)}}
\NewDocumentCommand{\dyadicV}{m m}{\set{V}_{#2}^{(#1)}}

\NewDocumentCommand{\wavelet}{o o d()}{\ensuremath{\varphi\IfValueT{#2}{_{#2}}\IfValueT{#1}{^{(#1)}}\IfValueT{#3}{\!\left(#3\right)}}}
\NewDocumentCommand{\CoeffSpace}{o}{\mathfrak{B}\IfValueT{#1}{^{(#1)}}}
\NewDocumentCommand{\coeff}{s o o}{
    \IfValueTF{#3}
        {\IfBooleanTF{#1}{\bar{b}}{b}_{#3}}
        {\bm{\IfBooleanTF{#1}{\bar{b}}{b}}}%
    \IfValueT{#2}{^{(#2)}}
}
\NewDocumentCommand{\coeffC}{s o o}{
    \IfValueTF{#3}
        {\IfBooleanTF{#1}{\bar{c}}{c}_{#3}}
        {\bm{\IfBooleanTF{#1}{\bar{c}}{c}}}%
    \IfValueT{#2}{^{(#2)}}
}
\NewDocumentCommand{\RadHaarSpace}{o}{\IfValueTF{#1}{\mathcal{G}_{#1}}{\mathcal{G}}}

\NewDocumentCommand{\saw}{m d()}{\ensuremath{\varphi^{(#1)}\IfValueT{#2}{\!\left(#2\right)}}}
\NewDocumentCommand{\walshMat}{m}{\ensuremath{\mt{W}^{(#1)}}}
\NewDocumentCommand{\walshVec}{m m o}{\ensuremath{\IfValueTF{#3}{\upsilon}{\vc{\upsilon}}^{(#1)}_{#2\IfValueT{#3}{,#3}}}}
\NewDocumentCommand{\walsh}{s m m d()}{\ensuremath{\IfBooleanTF{#1}{\bar{\omega}}{\omega}_{#3}^{(#2)}\IfValueT{#4}{\!\left(#4\right)}}}

\NewDocumentCommand{\argmin}{o}{
    \operatornamewithlimits{arg min}\IfValueT{#1}{_{#1}}
}

\NewDocumentCommand{\norm}{s m o}{
    \ensuremath{
        \IfBooleanTF{#1}{\|{#2}\|}{\left\|{#2}\right\|}
        \IfValueT{#3}{_{#3}}
    }
}
\NewDocumentCommand{\abs}{s m o}{
    \ensuremath{
        \IfBooleanTF{#1}{|{#2}|}{\left|{#2}\right|}
        \IfValueT{#3}{_{#3}}
    }
}
\NewDocumentCommand{\innerProd}{s m m o}{
    \ensuremath{
        \IfBooleanTF{#1}{\langle{#2},{#3}\rangle}{\left\langle{#2},{#3}\right\rangle}
        \IfValueT{#4}{_{#4}}
    }
}

\usepackage{lastpage}
\jmlrheading{--}{2026}{1-\pageref{LastPage}}{6/26; Revised -/--}{-/--}{00-0000}{Thomas Dittrich, Oliver Potocki, and Philipp Grohs}

\ShortHeadings{The Information-Theoretic Benefit of Shared Representations}{Dittrich, Potocki, and Grohs}
\firstpageno{1}

\begin{document}
\title{
The Information-Theoretic Benefit of Shared Representations under Orthogonality Constraints
}

\author{\name Thomas Dittrich \email thomas.dittrich@oeaw.ac.at \\
       \addr Johann Radon Institute of Computational and Applied Mathematics, Austrian Academy of Sciences
       \AND
       \name Oliver Potocki \email oliver.potocki@univie.ac.at \\
       \addr Faculty of Mathematics, University of Vienna
       \AND
       \name Philipp Grohs \email philipp.grohs@univie.ac.at \\
       \addr Johann Radon Institute of Computational and Applied Mathematics, Austrian Academy of Sciences\\
       \textnormal{and}\\
       Faculty of Mathematics, University of Vienna
       }

\editor{Editor}

\maketitle

\begin{abstract}

Modern deep learning architectures are increasingly multi-task and multi-modal, using a pretrained foundation model combined with task-specific, fine-tuned models.
This mirrors the classical paradigm of multitask learning.
Given a set of distinct regression problems, it is possible to solve each problem separately using scalar-valued models or solve them jointly with a single vector-valued model.
Empirically, exploiting similarity across the different problems can significantly improve overall performance. While the generalization and sample complexity properties of multitask learning have been widely studied, the rigorous investigation of the parametric complexity of joint approximation in comparison to separate approximation remains less well understood. The question is particularly relevant in modern deep learning, where models are increasingly required to satisfy structural constraints such as equivariance, conservation laws, or orthogonality. 

We prove lower and upper bounds on the description-length for separate and joint approximation classes, respectively, in the uniform norm. The key ingredient is a constructive rate separation. We build a class of orthogonal functions by composing a shared hard feature, realized by a Rademacher-Haar wavelet series, with Sawtooth-Walsh readouts that enforce orthogonality of output coordinates.
The dyadic tree structure of the Rademacher-Haar wavelet concentrates the approximation hardness in the common feature component, while the readouts act as task-specific heads. Using an information-theoretic framework, we obtain a sharp gap between the optimal approximation rates achievable by joint and separate coding.
Finally, we realize this separation in a neural network model using Heaviside activations via reduction to triangle-wave approximation.
Our results show that even under an orthogonality constraint joint approximation requires strictly fewer bits in compositional architectures, provided the tasks share a latent hard feature.
This provides theoretical insight into the description-length-efficiency of compositional multi-output architectures and clarifies how neural networks can retain expressivity under geometric constraints. 

\end{abstract}

\begin{keywords}
  Multitask Learning, Approximation Theory, Haar Wavelet, Neural Networks, Orthogonality Constraints
\end{keywords}

\section{Introduction}
\label{sec:intro}

Modern neural networks are increasingly used as shared-representation models, for example, in foundational model pipelines \citep{Bommasani22OpportunitiesRisksFoundation}, a large hidden representation is trained once and then used through task-specific heads and readout maps.
This principle appears in language models, where broad pretraining is followed by task-specific fine tuning \citep[see][]{Brown20LanguageModelsArea,Touvron23Llama2Open},
and in scientific machine learning, where equivariant message passing architectures for machine-learned interatomic potentials
construct shared atomic features from which energies, forces, and other quantities are derived \citep[see for example the MACE model of][]{Batatia22MACEHigherOrder}.
In particular, recent models utilize multi-head foundation models, where a single shared backbone maps the input to a latent representation, and several task specific heads turn the same representation into different task-specific readouts  \citep[see][]{Batatia25FoundationModelAtomistic,Batatia25CrossLearningElectronic}.
In scientific machine learning, constraints play an important role in ensuring the adherence to physical laws, giving rise to architectures that guarantee symmetry, equivariance, or orthogonality \citep{Cohen16GroupEquivariantConvolutional,Bronstein21GeometricDeepLearning}.
This then leads to the question of what the effect of such constraints is on the expressivity of multi-head models.

A related phenomenon occurs in operator learning, where architectures such as the DeepONet \citep{Lu21LearningNonlinearOperators}, the PCA-Net \citep{Bhattacharya21ModelReductionNeural}, or the Fourier neural operator \citep{Li21FourierNeuralOperator} are designed to map between function spaces.
These practical deep operator learning models use finite-dimensional shared output representations or implicit finite discretizations, for example, through shared basis or trunk functions.
This naturally leads to the classical finite-mode view point of principal component analysis (PCA) \citep{Jolliffe02PrincipalComponentAnalysis} and functional principal component analysis (fPCA) \citep{Ramsay05FunctionalDataAnalysis} as well as model order reduction.
The core of principal component analysis is built on the idea of selecting a finite dimensional subspace based on a variance criterion,
where orthogonality of the basis reduces redundancy of directions and stabilizes coefficient representations.
Orthogonality poses a geometric constraint, however, and makes no claim about statistical independence.
It does not preclude modes from being generated by a shared underlying structure.

Given the importance of these methods,
it becomes a straightforward question to ask about the information-theoretic effect of constrained joint approximation,
in our case encoded via an orthogonal finite bit representation.
We show that even for potentially nonlinear readouts an underlying shared feature space assumption does not diminish the effectiveness of sharing information between coordinates,
thereby also further shedding light on the effectiveness of deep learning.

\subsection{Contributions}
Our main focus for this work is on $M$-tuple of target functions
\begin{align}
F = (F_1, \ldots, F_M) \in L^\infty([0,1];\mathbb{R}^M)
\end{align}
and two general methods to approximate such functions.
Namely, we consider separate approximation of each component individually and joint approximation of all components of the vector-valued map $F$ combined.
In the first case, information cannot be shared between different coordinates,
while in the second case such a feature can be stored once and then reused in each mode.

In our analysis of approximation errors, we work in the uniform norm, rather than in an averaged $L^p$-metric.
The main reason for this choice is that in scientific machine learning and operator learning the learned surrogate represents a function on a physical domain and may be queried at arbitrary spatial or temporal locations.
Using the $L^\infty$-metric makes the error guarantee independent of a prescribed sampling measure or computational mesh,
which is particularly important when boundary conditions or structural constraints should not be discounted by averaging. The uniform norm is also the regime in which the theory-to-practice gap is most severe. For any sampling-based algorithm, the convergence rate in the $L^p$ norm is bounded by the Monte-Carlo rate of order $m^{-1/p}$ \citep{Grohs25TheorytoPracticeGapNeural}, so that in the uniform norm no nontrivial rate remains. Provable approximation rates can then no longer be realized from point samples \citep{Grohs24ProofTheorytoPracticeGap}. High accuracy comes at the cost of a sample size that grows exponentially in the architecture, for ReLU as well as $\tanh$ networks \citep{berner2023learning,Grohs26LimitationsLearningTanh}. In this regime, the only remaining measure of cost is the description complexity of the representation itself. We therefore compare joint and separate approximation in terms of this complexity. 
A separation for an orthogonal target is especially informative, because it shows that the advantage of joint approximation is not an artifact of distributional averaging,
mesh-dependent evaluation or linearly redundant output coordinates,
but solely caused by the reduction in information-complexity produced by encoding the shared latent feature once.

In view of pipelines for foundational models, the question we aim to answer is whether there exists a class of compositional orthogonal $M$-tuples of the form
\begin{align}
F_g = (f_1 \circ g, \ldots, f_M \circ g),
\end{align}
which are generated by one scalar hidden representation $g$ and orthogonal heads $f_i$,
such that each coordinate $f_i \circ g$ is individually hard to approximate,
for $i\neq j$ the coordinates $f_i \circ g$ are pairwise orthogonal,
and the vector $F_g$ can still be more efficiently approximated by a joint method.
This is a non-obvious existence question as the multiple requirements are competing. Orthogonality demands non-redundant coordinates, while a shared feature is redundancy by design. Indeed, the most direct way to enforce it, coordinates with disjoint supports, eliminates all shared information and admits no gap.
A similar tension arises for other output constraints such as equivariance or conservation laws, since any such constraint couples the output coordinates.
Among these orthogonality is the natural test case as the constraint is exactly verifiable by inner products.
We resolve it by moving the orthogonality constraint into the heads; the readouts are non-linear, not linear projections, and they decorrelate the outputs while hardness remains in the shared feature. 

We realize this with a function class combining Rademacher-Haar wavelet series with Sawtooth-Walsh readouts.
The Rademacher-Haar series was chosen because it provides a transparently hard target for uniform approximation as its coefficients form a dyadic tree of independent signs.
Our main result from \cref{thm:approximation_gap} can be informally stated as follows:
\begin{theorem}[Informal]
    Let $M\in\nn$ and $L\defeq\lceil\ld(M)\rceil$. Then there exists a class 
    \begin{align*}
        \mathcal{F} = \{(f_1 \circ g, \ldots, f_M \circ g)\colon g \in \set{G}\} \subset L^\infty([0,1];\mathbb{R}^M),
    \end{align*}
    where $\set{G}$ is a class of Rademacher-Haar wavelets and $f_i$ are Walsh-sawtooth functions of order $L$, such that each coordinate is orthogonal.
    Furthermore, with $\delta^{(S)}$ being an optimal separate $n$-bit codec for $\mathcal{F}$ and $\delta^{(J)}$ being the optimal joint $n$-bit codec,
    the gap between separate and joint codecs is
    \begin{align*}
        \frac{
            \left\|F-\delta^{(S)}_D\circ \delta_E^{(S)}(F)\right\|_{\infty}
        }{
            \left\|F-\delta_D^{(J)}\circ \delta_E^{(J)}(F)\right\|_{\infty}
        }
        \geq
        \frac{M}{4}\cdot \frac{1+\frac1n}{1+\frac{2M^2}{n}}
    \end{align*}
    for every $F\in\set{F}$ if $n=M(2^K-2^L)$ for some $K\in\nn$.
    In case that $n\neq M(2^K-2^L)$ for all $K\in\nn$, then there exists an $F\in\set{F}$ for which the gap holds.
    In particular, this joint codec dominates the best separate codec with asymptotic improvement of order $M / 4$.
\end{theorem}
The order of this gap reflects that a separate code must encode the shared feature anew for each of the $M$ coordinates, while a joint code encodes it only once.

To provide this information-theoretic treatment of orthogonal multitask targets, we consider the class of joint codecs simply as the encoder/decoder pairs acting on $L^{\infty}([0,1];\rr[M])$ directly.
We then introduce the class of separate codecs as a class of $M$ independent encoder/decoder pairs acting on $L^\infty([0,1];\rr)$ and show the existence of a class of orthogonal feature maps generated from one scalar Rademacher-Haar feature through fixed Sawtooth-Walsh readouts.
A compositional approach ensures that the full approximation hardness of the class is concentrated in the common feature.
This avoids the degenerate case where orthogonality carries no information-theoretic content that would arise from constructing this constraint via disjoint domains or a product approach.

We furthermore provide a precise treatment of the underlying Rademacher-Haar class in uniform norm.
First, we show that the truncation of the series represents an optimal finite-bit discretization at the respective bit-level.
Extending this to arbitrary bit budgets yields explicit upper and lower rate bounds.
This is superior to an approach based solely on metric entropy estimates, as we get exact constants that allow us to give a sharp separation between separate and joint encodings of this class.

Finally we return from the information theory based construction to neural network theory.
Here, Heaviside activations are the natural network model for realizing the Rademacher-Haar wavelet series in uniform norm due to the piecewise constant nature of the truncated series.
We then realize the composition with Sawtooth-Walsh functions via an approximation by triangle-waves and furthermore utilizing the construction of \cite{Yarotsky17ErrorBoundsApproximations} in order to realize pointwise multiplications. We construct an explicit Heaviside architecture with a shared trunk and $M$ fixed readouts. We provide a proof of the separation of one joint Heaviside architecture compared to $M$ separate architectures in \cref{thm:network_gap}.
This will show that the joint-separate approximation gap is not an artifact of a peculiar finite-bit discretization of one codec construction, but directly holds for an explicit neural network architecture.
Practically, this is the same principle that underlies shared trunks and foundation models, the expensive feature extractor is paid for once, while the heads implement task-specific readouts.

From the viewpoint of machine learning, these results give a rigorous treatment of a familiar but often informal heuristic.
It is widely observed that joint regression algorithms on vector-valued problems yield superior performance, a phenomenon usually attributed to tasks that share information.
Said phenomenon has gained particular prominence with the rise of deep learning, where internal latent representations are integral to the ability of neural networks to approximate complex targets.
To the best of our knowledge, our work provides the first rigorous mathematical treatment of this phenomenon from an approximation-theoretic perspective.

\subsection{Related work}
The present work is situated at the intersection of nonlinear approximation theory, multitask learning, random wavelet constructions, and orthogonal subspace methods. Our contribution combines these viewpoints in an information-theoretic approximation model. In particular, we construct an orthogonal vector-valued function class whose coordinates are individually hard to approximate, while a joint codec can still exploit a shared latent feature and thereby achieve a strictly better approximation rate.

\cite{Bolcskei19OptimalApproximationSparsely} and \cite{elbrachter2021deep} derive lower bounds on memory requirements and connectivity for approximating arbitrary function classes in $L^2$, uniformly over the class, and show that sparsely connected neural networks can optimally approximate broad classes of representation systems, including wavelets and central systems from harmonic analysis. With our Rademacher-Haar wavelet series we show that the representation system itself might contain a shared source of complexity depending on the encoding. Similarly, \cite{Petersen18OptimalApproximationPiecewise} and \cite{Yarotsky17ErrorBoundsApproximations} study the approximation rates of deep ReLU networks for piecewise and smooth Sobolev-smooth functions, respectively. In particular \cite{Petersen18OptimalApproximationPiecewise} show that for functions factorizing through a lower-dimensional feature map, the approximation rate is dominated by the dimension of the feature space, not the dimension of the ambient space. Their results, however, concern the expressivity of a function class, whereas our work in contrast is based on an information-theoretic codec framework to isolate the finite-bit gain of joint approximation.

\cite{Grohs23PhaseTransitionsRate} study rate-distortion theory in connection with approximation-theoretic phase transitions, compression rates, and generic sharpness of approximation results. Their work provides a natural theoretical background for finite-bit discretizations, minimax approximation rates, and optimal coding. However, their work focused on the general approximation complexity of a given function class for general representation systems. In contrast, our work differs in that, in addition to deriving a rate for a specific class of functions, we also study the effect of imposing a restriction on the structure of the encoder.

Multitask learning, as formulated by \cite{Caruana97MultitaskLearning}, aims to learn several tasks in parallel to utilize a shared latent representation so that different tasks can share information. This naturally leads to the formulation of multitask learning as the approximation of a structured class sharing a latent feature, thus motivating a compositional model of multitask learning where each singular task is precisely a different readout of the same feature.

The statistical theory of learning shared representations across tasks goes back to \cite{Baxter00ModelInductiveBias}, who formalized multitask learning as learning an inductive bias from a family of related tasks. \cite{Maurer16BenefitMultitaskRepresentation} study the benefit of multitask representation learning from a statistical learning perspective and show that shared representations improve learning when all tasks are generated by a common feature and the number of tasks is large. More recent work establishes provable sample complexity gains of learning a low-dimensional shared 
representation \citep{Tripuraneni21ProvableMetaLearningLinear, Du21FewShotLearningLearning}. \cite{Ruder17OverviewMultiTaskLearning} provides an overview of multitask learning in deep learning and discusses common mechanisms by which neural architectures share information such as hard and soft parameter sharing and representation-based transfer over tasks. These works already elucidate the core principle of multitask learning, in that shared latent features improve multiple learning tasks.

More recently, \cite{Shenouda24VariationSpacesMultioutput} studied multi-output through vector-valued variation spaces and showed that the corresponding regularization structure promotes shared features across outputs. Their analysis explains feature sharing as a consequence of the induced variation norm, which, for homogeneous activations, is closely related to weight decay and favors atoms that simultaneously contribute to multiple outputs. Our contribution is distinct: we isolate feature-sharing at the level of minimax finite-bit approximation, even before training or generalization enter the picture. In this setting the separation is not a statement about which atoms a regularizer prefers, but a sharp statement about the additional cost a separate model must pay for repeatedly encoding a shared latent feature that a joint description only pays for once. \cite{Weihs26MultipleNeuralOperators} established near-optimal approximation rates and a curse of parametric complexity lower bound for Lipschitz classes of multiple-operator maps. They conclude that multi-task operator learning follows the same scaling laws as single operator learning. In particular, they show that the exponential blow-up in constructive bounds for this setting is an artifact of the analysis and not an intrinsic feature of the multi-operator construction. Their classes do not enforce a shared latent feature, as the classes are related solely through the Lipschitz dependence on the operator descriptor. The hard instances behind their lower-bound are obtained by lifting scalar functionals that already exhibit the curse of rank-one multiple-operator maps.

Our work asks a different question. Rather than bounding the complexity of a class, we describe two regimes of describing the same object and quantify what separates them. For this, we perform an information-theoretic treatment of an explicit constructive function class and explore the effective boundary for hard tasks. In particular, we demonstrate that orthogonality as a geometric constraint on that constructive class is not a prohibitive factor for the effectiveness of multitask learning.

Random wavelet series provide the specific constructive tool used in this paper. Optimal finite-bit encoding of deterministic wavelet expansions goes back to \cite{Cohen01TreeApproximationOptimal}, who construct tree-based encoders matching the Kolmogorov entropy of Besov balls. \cite{Esser24RegularityPropertiesRandom} study regularity properties of random wavelets series obtained by multiplying deterministic wavelet coefficients by unbounded i.i.d. random variables, in particular, they show that for almost every continuous function this randomization is almost surely nowhere locally bounded. \cite{Horst24BesovRegularityRandom} study Besov regularity of random wavelet series, motivated by the Bayesian practice to use random stochastic processes as priors in inverse problems. These works are mainly concerned with regularity properties of random harmonic analysis representation systems without deriving approximation bounds or rates. Randomizing wavelet coefficients gives rise to function classes with interesting microscale behavior, making them a natural fit for scale-based complexity studies. Based on this, we, in contrast to previous works, construct a Rademacher distributed Haar-wavelet series whose unresolved tail has the same uniform norm on every dyadic subset, at every scale. Thus, the unresolved tail therefore carries the majority of the approximation hardness. Similar finite-scale packing arguments appear in the entropy and minimax lower bounds of \cite{Parhi23NearMinimaxOptimalEstimation} and \cite{Siegel24SharpBoundsApproximation}, where a critical resolution is used to certify a rate. The class of Rademacher-Haar wavelets may be read as the extremal object of the first-order analogue of the Radon-BV geometry that \cite{Parhi21BanachSpaceRepresenter} associate with shallow neural networks trained with weight-decay. In this comparison, the corresponding elementary atoms are the Heaviside neurons in \cref{sec:heaviside_approx}. At level $k$, the Rademacher-Haar feature has $2^k$ coefficients all of magnitude $2^{-k-1}$. Hence the $\ell^1$ budget of the coefficients at this level is $1/2$, independent of signs. In this limited sense, our class is then best understood as a boundary-type model for the corresponding variation geometry, not as a bounded subset of the associated variation space in the sense of \cite{Siegel23CharacterizationVariationSpaces}.

Approximation by linear subspaces is one of the most common problems in machine learning. It is straightforward to see that a linear approximation has benefits such as interpretability or ease of computation. Perhaps the most popular technique in dimensionality reduction is principal component analysis (PCA) \citep{Jolliffe02PrincipalComponentAnalysis}, which aims to reduce the complexity of data by finding the exact n-dimensional subspace that explains the most variance. A functional extension of this idea has become popular in recent years by functional principal component analysis (fPCA) \citep{Ramsay05FunctionalDataAnalysis}. In both approaches, orthogonality is key to ensure optimality of the finite-dimensional reduction and stability of the coefficients of the representation. 

%%%%%%%%%%%%%%%%%%%%%%%%%%%%%%%%%%%%%%%%%%%%%%%%%%
%%%%%%%%%%%%%%%%%%%%%%%%%%%%%%%%%%%%%%%%%%%%%%%%%%
%%%%%%%%%%%%%%%%%%%%%%%%%%%%%%%%%%%%%%%%%%%%%%%%%%
\section{Preliminaries and Notation}

A central element of our analysis is discretization to $n\in\nn$ bits, that is, to an element in $\{0,1\}^n$.
This is a natural setting in information theory and can be formalized via the following classes of encoder/decoder pairs and the corresponding distortion measure:
\begin{definition}[$n$-Bit Codecs and Distortion]
\label[definition]{def:prelim_codec}
    Let $\fSpace{X}$ be a Banach space, $\set{S}\subset\fSpace{X}$, $n\in\nn$.

    We define the classes of $n$-bit encoding/decoding pairs as
    \begin{align}
        \codecClass{\set{S}}{\fSpace{X}}[n]
        \defeq 
        \{(\encoder[n],\decoder[n])|\encoder[n]\deffun \set{S}\to\{0,1\}^n,\,\decoder[n]\deffun \{0,1\}^n\to\fSpace{X}\}
    \end{align}
    and the set of all codecs as the sequences of encoder/decoder pairs over $n\in\nn$
    \begin{align}
        \codecClass{\set{S}}{\fSpace{X}}\defeq\bigtimes_{n\in\nn}\codecClass{\set{S}}{\fSpace{X}}[n].
    \end{align}
    Furthermore, define the best distortion with encoder/decoder pairs from $\codecClass{\set{S}}{\fSpace{X}}[n]$ as
    \begin{align}
        \distortion{\set{S}}{\fSpace{X}}[n]
        \defeq
        \inf_{(\encoder,\decoder)\in \codecClass{\set{S}}{\fSpace{X}}[n]} \sup_{x\in\set{S}}\norm{x-\decoder\circ\encoder(x)}[\fSpace{X}].
    \end{align}
\end{definition}

Although the formulation in encoder/decoder pairs is natural in information theory, our work is not focused on the actual representation in terms of bits.
Rather, we are interested in the composition of the encoder and decoder as a map from $\set{S}$ to $\fSpace{X}$, where the number of bits $n$ corresponds to the cardinality of the image of the decoder.
We can therefore equivalently consider $B$-bit discretization maps, which are defined via the image set $\set{C}$ of the decoder as follows.
\begin{definition}[Discretization Maps]
\label[definition]{def:disc_maps}
    Let $\fSpace{X}$ be a Banach space.
    A map $\Delta \deffun \fSpace{X}\to \fSpace{X}$ is called a $B$-bit discretization map if $\card{\operatorname{range}(\Delta)}=2^B$ for some $B\in\nn_0$.
    For a fixed $B\in\nn_0$ and $\set{C}\subset \fSpace{X}$ with $\card{\set{C}}=2^B$ we define the projective $B$-bit discretization map $\Delta_\set{C}\deffun \fSpace{X}\to\fSpace{X}$ as
    \begin{align}
        \Delta_\set{C}(x)\defeq\argmin_{y\in\set{C}}\norm{x-y}[\fSpace{X}].
    \end{align}
\end{definition}
Note that the $\argmin$ can in general result in ties, which have to be resolved in some way.
The precise way of resolving ties is not important for the following statements and we will omit further discussions.

For our results on the Rademacher-Haar wavelet series, we often require to partition the domain into a family of $2^L$ disjoint intervals for any $L\in\nn$.
Specifically, we use the following notation for two types of these dyadic sets.
\begin{definition}[Dyadic Subsets]
    For $L\in\nn_0$ and $i\in[2^L]$ we introduce the notation
    \begin{align}
        \dyadicU{L}{i}\defeq 2^{-L}[i,i+1),\qquad
        \dyadicV{L}{i}\defeq 2^{-L+1}[i,i+1)-1
    \end{align}
    as the dyadic subsets of $\set{U}\defeq[0,1)$ and $\set{V}\defeq[-1,1)$.
\end{definition}
Throughout this work, we write $\ld := \log_2$ for the binary logarithm, use the index set $\nn_{-1}:= \{-1, 0, 1, 2, \dots\}$, and consider orthogonality in the $L^2$-sense.
Finally, in order to simplify notation for indices, we use the notation
\begin{align}
    [M]\defeq \{0,\ldots,M-1\}.
\end{align}

%%%%%%%%%%%%%%%%%%%%%%%%%%%%%%%%%%%%%%%%%%%%%%%%%%
%%%%%%%%%%%%%%%%%%%%%%%%%%%%%%%%%%%%%%%%%%%%%%%%%%
%%%%%%%%%%%%%%%%%%%%%%%%%%%%%%%%%%%%%%%%%%%%%%%%%%
\section{Information-Theoretic Approximation Gap}
The main theorem of the present paper is on the existence of a function class, which has a non-trivial complexity and a corresponding orthogonal feature map for which we show a gap between joint and separate coding.
In this section, we first introduce the two coding schemes in detail, the initial class of Rademacher-Haar wavelets and Sawtooth-Walsh functions as our orthogonal feature maps.
Based on these definitions, we present our main result in \cref{thm:approximation_gap}.

In \cref{sec:rad_haar_properties} we then show tight upper and lower bounds on the minimax discretization rates for the Rademacher-Haar wavelet series and some properties of the Sawtooth-Walsh orthogonal feature maps.
We then use these properties in \cref{sec:sep_joint_bounds} in order to prove a lower bound for separate discretization and an upper bound for joint discretization.
In \cref{sec:proof_main_theorem} we combine the previous result and prove \cref{thm:approximation_gap}.

We distinguish between joint and separate coding.
The main idea here is that joint coding can access all output dimensions of a multivariate function class when discretizing, whereas separate coding can only access a single output dimension at a time and needs to discretize the dimensions independently without any cross-information.
\begin{definition}[Joint and Separate Coding]
    Let $M\in\nn$, $\set{H}\subset L^\infty([0,1];\rr[M])$, and for $i\in[M]$ define the marginal sets
    \begin{align}
        \set{H}_i\defeq\{h_i|(h_0,\ldots,h_{M-1})\in\set{H}\}.
    \end{align}
    For every $n\in\nn[M]$ and $i\in[M]$, 
    we then consider the marginal encoder/decoder pairs $\codecClass{\set{H}_i}{L^\infty([0,1])}[n]$ as in \cref{def:prelim_codec}
    and define the class of separate encoder/decoder pairs as
    \begin{align}
        \sepClass[n]\defeq \left\{
                \left((\encoder<0>[n_0],\ldots,\encoder<M-1>[n_{M-1}]),(\decoder<0>[n_0],\ldots,\decoder<M-1>[n_{M-1}])\right)
            \middle|
                (\encoder<j>[n],\decoder<j>[n])\in\codecClass{\set{H}_j}{L^\infty([0,1])}[n]
            \right\}.
    \end{align}
    For every $N\in\nn$, the class of joint encoder/decoder pairs is defined via \cref{def:prelim_codec} as
    \begin{align}
        \jointClass[N]\defeq \codecClass{\set{H}}{L^\infty([0,1];\rr[M])}.
    \end{align}
\end{definition}
Note that with this definition we consider 
\begin{align}
    \encoder<\mathrm{S}>[n]=(\encoder<0>[n_0],\ldots,\encoder<M-1>[n_{M-1}])\quad\text{and}\quad\decoder<\mathrm{S}>[n]=(\decoder<0>[n_0],\ldots,\decoder<M-1>[n_{M-1}])
\end{align}
as maps
\begin{align}
     \encoder<\mathrm{S}>[n]\deffun \bigtimes_{i\in[M]}\set{H}_i\to {\{0,1\}^{\sum_{i=0}^{M-1} n_i}}\quad\text{and}\quad\decoder<\mathrm{S}>[n]\deffun{\{0,1\}^{\sum_{i=0}^{M-1} n_i}}\to L^\infty([0,1])^M,
\end{align}
respectively.

For our initial function class that represents the shared feature,
we consider random wavelet series that are based on the Haar wavelet and use Rademacher random variables as weights.
Note that we consider minimax rates, which means that our later results are not in expectation and we therefore do not require a definition for a probability measure on the function class.
\begin{definition}[Rademacher-Haar Wavelet Series]
    Let $k,L\in\nn_0$ and $i\in[2^k]$.
    The mother wavelet of the Haar wavelet is given by
    $\wavelet=\chi_{[0,1/2)}-\chi_{[1/2,1)}$
    and its $i$-th child wavelet on level $k$ is given by
    $\wavelet[k][i]=2^{-k-1}\wavelet(2^k\cdot-i)$.
    We denote the parameter space of the Rademacher-Haar wavelet series as
    $\CoeffSpace=\times_{l=0}^\infty\{\pm 1\}^{2^l}$ and for any $\coeff\in\CoeffSpace$ we consider the realization functions
    \begin{align}
        h_k(\cdot;\coeff)&\defeq\sum_{m=0}^{2^k-1}\coeff[k][m]\wavelet[k][m](\cdot),
        \qquad
        g_L(\cdot;\coeff)\defeq\sum_{l=0}^Lh_l(\cdot;\coeff),
        \qquad\text{and}\qquad
        g(\cdot;\coeff)\defeq \lim_{L\to\infty}g_L(\cdot;\coeff).
    \end{align}
    The class of Rademacher-Haar wavelet series is then given by
    \begin{align}
        \RadHaarSpace
        \defeq
        \left\{
        g(\cdot;\coeff)
        \middle|
        \coeff\in\CoeffSpace
        \right\}
    \end{align}
    and the class of $L$-level truncated Rademacher-Haar wavelet series with $L\in\nn_{-1}$ by
    \begin{align}
        \RadHaarSpace[L]
        \defeq
        \left\{
        g_L(\cdot;\coeff)
        \middle|
        \coeff\in\CoeffSpace[L]
        \right\}.
    \end{align}
    For an arbitrary $g\in\RadHaarSpace$ we use the notation $g_L\in\RadHaarSpace[L]$ to denote the corresponding $L$-level truncation. 
\end{definition}
Random wavelet series have a very interesting property: the random behavior on each dyadic subset is similar to all the other subsets.
That is, if we look at dyadic subsets of length $2^{-L}$ and only look at the tail of the series starting from index $L$, then the random behavior on different subsets is indistinguishable and requires the same complexity for discretization.
This can be seen in \cref{fig:rad_haar_wavelet}.
\begin{figure}
    \centering
    \includegraphics[page=1]{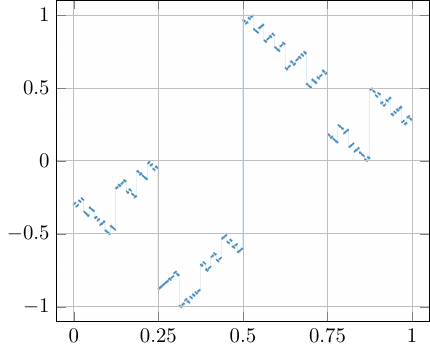}
    \hfill
    \includegraphics[page=2]{figures/haar_wavelet/function_plot.pdf}
    \caption{An example of the Rademacher-Haar wavelet series (left) together with the tail starting from index $L=2$ (right).}
    \label{fig:rad_haar_wavelet}
\end{figure}
Another observation from \cref{fig:rad_haar_wavelet} for the Rademacher-Haar wavelet series is that restricting the function to a dyadic subset also restricts the image of the function to a dyadic subset.
That is, composing the wavelet series with a sawtooth function is essentially the same as multiplying with a characteristic function and restricting to the tail of the wavelet.
A simple idea for orthogonalization would therefore be to apply this transformation for an appropriate number of subsets and obtain orthogonality by having non-overlapping supports.
However, this approach would result in the same discretization performance in the joint and the separate approaches as there is no shared information between the orthogonal features.

Instead, we resort to the following combination of Sawtooth waves and Walsh functions.
This construction keeps the complexity of the full tail on every dyadic subset, while strictly enforcing orthogonality.
\begin{definition}[Sawtooth-Walsh functions]
    Let $L\in\nn_0$ and $\walshMat{L}\in\{\pm 1\}^{2^L\times 2^L}$ be a fixed Hadamard matrix and denote its columns by $\{\walshVec{L}{i}\in\{\pm 1\}^{2^L}|i\in[2^L]\}$.
    We define the sawtooth function as
    \begin{align}
        \psi_0&:\rr \to \rr, & \psi_0(x)&\defeq x - \lfloor x\rfloor - 1/2,\\
        \saw{L}&:\rr \to \rr, & \saw{L}(x)&\defeq2\sqrt{3}\psi_0(2^L(x+1)/2),
    \end{align}
    and the Walsh functions of order $L$ that are induced by $\walshMat{L}$ as% the set $\{\walsh*{L}{i}:\dyadicU{0}{0}\to\{\pm 1\},\,i\in[2^L]\}$ such that for all $i,j\in[2^L]$ we set
    \begin{align}
        \walsh*{L}{i}:\dyadicU{0}{0}\to\{\pm 1\},
        \qquad
        \walsh*{L}{i}(x)=\walshVec{L}{i}[j]\quad\text{for all}\quad x\in\dyadicU{L}{j}\quad\text{and}\quad i,j\in[2^L].
    \end{align}
    We furthermore define the stretched Walsh functions induced by $\walshMat{L}$ as 
    \begin{align}\walsh{L}{i}\deffun \dyadicV{0}{0}\to \{\pm 1\},\qquad\walsh{L}{i}(x)\defeq \walsh*{L}{i}((x+1)/2)\quad\text{for all}\quad i\in[2^L].\end{align}
    Finally, for $M\in\nn$ with $2^{L-1}< M\leq 2^L$, we define the Walsh-Sawtooth functions of order $L$ as $(\saw{L}\cdot\walsh{L}{0},\ldots,\saw{L}\cdot\walsh{L}{M-1})$.
\end{definition}

Based on the class of Rademacher-Haar wavelet series and the Walsh-Sawtooth functions, we can now state our main theorem as follows.
\begin{theorem}[Multitask-Gap]
\label{thm:approximation_gap}
    Let $\set{G}\subset L^\infty([0,1])$ be the class of Rademacher-Haar wavelets, 
    $M\in\nn$ be the number of orthogonal features,
    $n\in\nn[M]$ with $n_i\geq 2$ be the number of bits in separate coding,
    $N=M\min_{i\in[M]}n_i$ be the number of bits in joint coding,
    $(f_0,\ldots,f_{M-1}):\rr\to\rr$ be the Sawtooth-Walsh functions of order $L\defeq \lceil\ld(M)\rceil$,
    and $\set{F}\defeq \{(f_0\circ g,\ldots,f_{M-1}\circ g)|g\in\set{G}\}$ be the class of orthogonal target functions. 
    Then, for any $n\in\nn[M]$ and $N=M\min_{i\in[M]}n_i$, the optimal joint coding scheme
    \begin{align}
        \codec<\mathrm{joint}>[N]&=(\encoder<\mathrm{J}>[N],\decoder<\mathrm{J}>[N])
        =
        \argmin_{(\encoder,\decoder)\in\jointClass[N]}
        \sup_{h\in\set{F}}
        \norm{h-\decoder\circ\encoder(h)}[L^\infty([0,1];\rr[M])],
        \intertext{and the optimal separate coding scheme}
        \codec<\mathrm{sep}>[n]&=(\encoder<\mathrm{S}>[n],\decoder<\mathrm{S}>[n])
        =
        \argmin_{(\encoder,\decoder)\in\sepClass[n]}
        \sup_{h\in\set{F}}
        \norm{h-\decoder\circ\encoder(h)}[L^\infty([0,1];\rr[M])],
    \end{align}
    exhibit the gap
    \begin{align}
        \frac{\norm{h-(\decoder<\mathrm{S}>[n]\circ\encoder<\mathrm{S}>[n])(h)}[L^\infty([0,1])]}
        {\norm{h-(\decoder<\mathrm{J}>[N]\circ\encoder<\mathrm{J}>[N])(h)}[L^\infty([0,1])]}
        \geq \frac{M}{4}\cdot\frac{1+\frac{1}{N}}{1+\frac{2M^2}{N}},
    \end{align}
    for all $h\in\set{F}$ if $N/M=2^K-2^L$ for some $K\in\nn$ and for some $h\in\set{F}$ if $N/M\neq 2^K-2^L$ for all $K\in\nn$.
\end{theorem}
Note that by choosing $N=M\min_{i\in[M]}n_i$ bits for joint coding, we indeed allocate at least as many bits to separate coding as to joint coding, while still obtaining a better performance in joint coding.

In order to prove this theorem, we first develop some properties of the Rademacher-Haar wavelet series in \cref{sec:rad_haar_properties},
Second, we derive a lower bound on the distortion for separate coding and an upper bound for the distortion in joint coding in \cref{sec:sep_joint_bounds}.
As the third and last step, we then combine these bounds to prove \cref{thm:approximation_gap} in \cref{sec:proof_main_theorem}.

\subsection{Properties of the Rademacher-Haar Wavelets and Sawtooth-Walsh Functions}
\label{sec:rad_haar_properties}

A major component in the analysis of the Rademacher-Haar wavelet is the geometric structure of the tails of the series.
In our first lemma we show that these tails have exactly the same supremum on every dyadic set of level $L+1$ for every member of the class of functions.
\begin{lemma}[Tail Supremum]
\label[lemma]{thm:truncation_error}
    Let $L\in\nn_{-1}$ and $j\in[2^{L+1}]$.
    Then for any $g\in\RadHaarSpace$ and the corresponding $L$-level truncation $g_L$, with $g_{-1}$ considered to be the zero function, we get
    \begin{align}
        \operatornamewithlimits{ess\,sup}_{x\in\dyadicU{L+1}{j}}g(x)-g_{L}(x)
        =
        \operatornamewithlimits{ess\,sup}_{x\in\dyadicU{L+1}{j}}g_{L}(x)-g(x)
        =
        2^{-L-1}.
    \end{align}
\end{lemma}
The idea for the proof is that each element of the series has a positive contribution on half of its support.
Depending on the sign of the corresponding coefficient, this is either the first half or the second half.
By following the indices of subsets where the contribution is positive, we construct a sequence that converges to the supremum and by the convergence of the lower bound to the upper bound we then show the precise value of the supremum.
\begin{proof}%
    \NewDocumentCommand{\ip}{m}{\hat{i}_{#1}} % index of dyadic set with coefficients +1
    \NewDocumentCommand{\yp}{m}{\hat{y}_{#1}}
    \NewDocumentCommand{\zp}{m}{\hat{z}_{#1}}
    \NewDocumentCommand{\Up}{m}{\hat{\set{U}}_{#1}}
    \NewDocumentCommand{\im}{m}{\check{i}_{#1}} % index of dyadic set with coefficients -1
    \NewDocumentCommand{\ym}{m}{\check{y}_{#1}}
    Let $\coeff\in\CoeffSpace$ and define the sequence $\{\hat{i}_l\}_{l=L+1}^\infty$ of ``up-indices'' recursively, starting with
    $\ip{L+1}\defeq j$ and iterating
    \begin{align}
        \ip{l}&\defeq 
        \begin{cases}
            2\ip{l-1},&\coeff[l-1][\ip{l-1}]=1,\\
            2\ip{l-1}+1,&\coeff[l-1][\ip{l-1}]=-1
        \end{cases}
    \end{align}
    for every $l\geq L+2$.
    We then observe for all $k,l\in\nn$ with $k>l\geq L+2$ that
    \begin{align}
        \coeff[l-1][\ip{l-1}]\wavelet[l-1][\ip{l-1}](\ip{l}2^{-l})=2^{-l}
        \quad\text{and}\quad
        \dyadicU{k}{\ip{k}}\subseteq\dyadicU{l}{\ip{l+1}}\subset\dyadicU{l}{\ip{l}}\subseteq\dyadicU{L+1}{j}.
    \end{align} 
    Furthermore, the $l$-th level of the Haar wavelet is constant on $\dyadicU{l+1}{\ip{l+1}}$,
    thus,
    \begin{align}
        h_l\left(\ip{k}2^{-k}\right)
        =
        \coeff[l][\ip{k}]\wavelet[l][\ip{k}]\left(\ip{k}2^{-k}\right)
        =
        \coeff[l][\ip{k}]\wavelet[l][\ip{l}](\ip{l+1}2^{-l-1})=2^{-l-1}.
    \end{align}
    With the above considerations it follows for all $k>L+2$ that
    \begin{align}
        g(\ip{k}2^{-k})-g_{L}(\ip{k}2^{-k})
        &=
        \sum_{l=L+1}^\infty h_l(\ip{k}2^{-k})
        =
        \sum_{l=L+1}^{k-1} 2^{-l-1}
        +
        \sum_{l=k}^\infty h_l(\ip{k}2^{-k})
        \geq
        2^{-L-1}-2^{-k+1}
    \end{align}
    and trivially also the uniform upper bound $g(y)-g_{L}(y)=\sum_{l=L+1}^\infty h_l(y)\leq 2^{-L-1}$ for all $y\in\dyadicU{L+1}{j}$.
    Therefore,
    \begin{align}
        2^{-L-1}
        \geq 
        \sup_{y\in\dyadicU{L+1}{j}}g(y)-g_{L}(y)
        \geq
        \lim_{k\to\infty}g(\ip{k}2^{-k})-g_{L}(\ip{k}2^{-k})
        =2^{-L-1}.
    \end{align}
    The supremum for $-(g-g_L)$ follows analogously with
    \begin{align}
        \im{l}\defeq
        \begin{cases}
            2\im{l-1}+1,&\coeff[l-1][\im{l-1}]=1,\\
            2\im{l-1},&\coeff[l-1][\im{l-1}]=-1.
        \end{cases}
    \end{align}
    For the essential supremum we observe that for all $l>L$ and all $x\in\dyadicU{l}{\ip{l}}$ that
    \begin{align}
        \inf_{y\in\dyadicU{l}{\ip{l}}}g(y)-g_L(y)
        &=
        -g_L(x)+g_{l-1}(x)-\sup_{y\in\dyadicU{l}{\ip{l}}}g_{l-1}(y)-g(y)
        \\&
        =
        -g_L(x)
        +\sup_{y\in\dyadicU{l}{\ip{l}}}\big(g(y)-2^{-l}\big)
        -\sup_{y\in\dyadicU{l}{\ip{l}}}\left(g_{l-1}(y)-g(y)\right)
        \\&
        =
        \sup_{y\in\dyadicU{L+1}{j}}g(y)
        -g_L(y)
        -2^{-l+1}.
    \end{align}
    That is, for every $\varepsilon>0$ we can find a half-open neighborhood of the supremum with non-zero volume such that $g(x)-g_L(x)$ is within $\varepsilon$ of the supremum.
\end{proof}

Next we show a lower bound on the distance between two realizations of the wavelet series.
Essentially, this lower bound just depends on the first level of the wavelet series on which the coefficients differ.
\begin{lemma}[Minimum Distance of Wavelet Series]
    \label[lemma]{thm:wavelet_distance_lower}
    Let $L_0\in\nn_0$, $j\in[2^{L_0}]$, and
    $\coeff,\coeffC\in\CoeffSpace$ be arbitrary but fixed
    and set\footnote{Note that we use the convention $\min \emptyset=-\max \emptyset=\infty$.}
    $L\defeq\min\{l\in\nn_0|\exists x\in\dyadicU{L_0}{j}:g_l(x;\coeff)\neq g_l(x;\coeffC)\}$,
    then
    \begin{align}
        \norm{g(\cdot;\coeff)-g(\cdot;\coeffC)}[L^\infty\!\left(\dyadicU{L_0}{j}\right)]\geq 2^{-\max\{L,L_0-1\}}.
    \end{align}
\end{lemma}

\begin{proof}
 In case $L=\infty$ we see that the right-hand side is $0$ and therefore the statement is trivially true.
    For $L\geq L_0$ there is an $i\in\{j2^{L-L_0},\ldots,(j+1)2^{L-L_0}-1\}$ such that $\coeff[L][i]\neq\coeffC[L][i]$.
    We define
    \begin{align}
        k\defeq\begin{cases}
            2i,&L\geq L_0,\\
            j,&L<L_0,
        \end{cases}
        \qquad\text{and}\qquad
        K\defeq \max\{L,L_0-1\}.
    \end{align}
    In both cases we see that $h_l(x;\coeff)$ and $h_l(x;\coeffC)$ are constant over $\dyadicU{K+1}{k}$ for all $l\leq K$ (note that this includes $l=L$).
    W.l.o.g. we assume that $h_L(x;\coeff)=-h_L(x;\coeffC)=2^{-L-1}$ for $x\in\dyadicU{K+1}{k}$ (the sign-reversed case follows analogously).
    By using \cref{thm:truncation_error} and bounding the remaining sum with the case of all coefficients being $-1$ we get
    \begin{align}
        \sup_{x\in\dyadicU{K+1}{k}}g(x;\coeff)-g_{L-1}(x;\coeff)
        &=
        \sup_{x\in\dyadicU{K+1}{k}}g(x;\coeff)-g_{K}(x;\coeff)+h_{L}(x;\coeff)+\sum_{l=L+1}^{K}h_{l}(x;\coeff)
        \\&
        \geq
        2^{-K-1}+2^{-L-1}-\sum_{l=L+1}^{K}2^{-l-1}
        =2^{-K}
    \end{align}
    and by the initial assumption $h_L(x;\coeffC)=-2^{-L-1}$ and bounding the tail we get
    \begin{align}
        \sup_{x\in\dyadicU{K+1}{k}}g(x;\coeffC)-g_{L-1}(x;\coeffC)
        =
        \sup_{x\in\dyadicU{K+1}{k}}\sum_{l=L}^\infty h_l(x;\coeffC)
        \leq
        -2^{-L-1}+\sum_{l=L+1}^\infty2^{-l-1}
        = 0.
    \end{align}
    Combined, this results in
    \begin{align}
        \norm{g(\cdot;\coeff)-g(\cdot;\coeffC)}[{L^\infty\left(\dyadicU{L_0}{j}\right)}]
        &\geq % we reduce the size of the set
        \norm{g(\cdot;\coeff)-g(\cdot;\coeffC)}[{L^\infty(\dyadicU{K+1}{k})}] % reduce the domain
        =
        \sup_{x\in\dyadicU{K+1}{k}} g(x;\coeff)-g(x;\coeffC)
        \\&
        =
        \sup_{x\in\dyadicU{K+1}{k}} g(x;\coeff)-g_{L-1}(x;\coeff)-\underbrace{(g(x;\coeffC)-g_{L-1}(x;\coeffC))}_{\leq 0}
        \\&
        \geq
        \sup_{x\in\dyadicU{K+1}{k}} g(x;\coeff)-g_{L-1}(x;\coeff) % replace terms with \coeffC by the supremum
        =2^{-K}.
    \end{align}
\end{proof}

Initially, we claimed that the class $\set{G}$ is non-trivial.
As a first step towards this claim, we now show a lower bound for all possible discretization schemes.

\begin{lemma}[General Discretization Bound]
\label[lemma]{thm:discretization_map_bound}
    Let $L,L_0\in\nn_0$, $j\in[2^{L_0}]$, and $\set{C}\subset L^\infty(\dyadicU{L_0}{j})$ with $\card{\set{C}}\leq 2^{2^{\max\{L-L_0,0\}+1}+\min\{L,L_0\}-1}$,
    then
    \begin{align}
    \sup_{g\in\RadHaarSpace}\norm{g-\Delta_\set{C}(g)}[{L^\infty(\dyadicU{L_0}{j})}]\geq 2^{-\max\{L,L_0-1\}-1}.
    \end{align}
\end{lemma}

\begin{proof}
    Let $L,L_0\in\nn_0$ and $\set{C}\subset L^\infty(\dyadicU{L_0}{0})$ with $\card{\set{C}}<\infty$ such that
    \begin{align}
    \varepsilon\defeq \sup_{g\in\RadHaarSpace}\norm{g-\Delta_\set{C}(g)}[{L^\infty(\dyadicU{L_0}{0})}]< 2^{-\max\{L,L_0-1\}-1}.
    \end{align}
    We show that the assumption $\varepsilon<2^{-\max\{L,L_0-1\}-1}$ implies $\card{\set{C}}>2^{2^{\max\{L-L_0,0\}+1}+\min\{L,L_0\}-1}$.
    The generalization to the domain $\dyadicU{L_0}{j}$ for an arbitrary $j\in[2^{L_0}]$ follows immediately.

    For $L_0>0$ we observe that the realization map $g(\cdot;\coeff)$ results in the identical function for some different elements $\coeff\in\CoeffSpace$ when restricting the domain to $\dyadicU{L_0}{0}=[0,2^{-L_0})$.
    We see that $\supp{\wavelet[l][m]}\cap \dyadicU{L_0}{0}=\dyadicU{l}{m}\cap \dyadicU{L_0}{0}\neq \emptyset$ iff $l<L_0$ and $m=0$ or $l\geq L_0$ and $\sup \dyadicU{l}{m}=(m+1)2^{-l}\leq 2^{-L_0}=\sup \dyadicU{L_0}{0}$.
    For all $L_0\in\nn_0$ and $x\in \dyadicU{L_0}{0}$ we can thereby reduce the inner sum in the wavelet series and get
    \begin{align}
        g(x;\coeff)=\sum_{l=0}^\infty\sum_{m=0}^{\max\{0,2^{l-L_0}-1\}}\coeff[l][m]\wavelet[l][m](x),
    \end{align}
    which results in the reduced coefficient sets
    \begin{align}
        \CoeffSpace_{L_0}&\defeq \times_{l=0}^\infty\{\pm 1\}^{\max\{1,2^{l-L_0}\}}
        \qquad\text{and}\qquad
        \CoeffSpace[L]_{L_0}\defeq \times_{l=0}^L\{\pm 1\}^{\max\{1,2^{l-L_0}\}}.
    \end{align}
    The space $\CoeffSpace[L]_{L_0}$ can then be identified with $\{\pm 1\}^d$ where 
    \begin{align}
        d=\sum_{l=0}^L\max\{1,2^{l-L_0}\}=\min\{L,L_0\}+2^{\max\{L-L_0,0\}+1}-1.
    \end{align}
    That is, it contains $2^d$ elements.
    Let $\coeff*_k, k\in[2^d]$ be an enumeration of $\CoeffSpace[L]_{L_0}$.
    For each $k\in[2^d]$ we select an extension $\coeff_k\in\CoeffSpace_{L_0}$ such that $\coeff[l][k,i]=\coeff*[l][k,i]$ for $l\in[L+1],\,i\in[\max\{1,2^{l-L_0}\}]$ and $\coeff[l][k,i]=1$ for $l>L,\,i\in[\max\{1,2^{l-L_0}\}]$.
    Then with \cref{thm:wavelet_distance_lower}, $\norm{g(\cdot,\coeff_{k_1})-g(\cdot,\coeff_{k_2})}[{L^\infty(\dyadicU{L_0}{0})}]\geq 2^{-\max\{L,L_0-1\}}>2\varepsilon$ for $k_1,k_2\in[2^d]$ with $k_1\neq k_2$.
    Furthermore, we select $\coeffC_0\in\CoeffSpace_{L_0}$ such that $\coeffC[l][0,i]=\coeff[l][0,i]$ for $l\in[L+1],\,i\in[\max\{1,2^{l-L_0}\}]$ and $\coeffC[l][0,i]=-1$ for $l>L,\,i\in[\max\{1,2^{l-L_0}\}]$.
    Then again with \cref{thm:wavelet_distance_lower}, $\norm{g(\cdot;\coeff_k)-g(\cdot;\coeffC_0)}[{L^\infty(\dyadicU{L_0}{0})}]\geq 2^{-\max\{L,L_0-1\}}>2\varepsilon$ for $k\neq 0$ and additionally
    \begin{align}
        \norm{g(\cdot;\coeff_0)-g(\cdot;\coeffC_0)}[{L^\infty(\dyadicU{L_0}{0})}]
        &\geq
        \abs{g(0;\coeff_0)-g(0;\coeffC_0)}
        =
        \abs{\sum_{l=0}^\infty \left(\coeff[l][0,0]-\coeffC[l][0,0]\right)2^{-l-1}}
        \\&
        =
        \abs{\sum_{l=L+1}^\infty (1-(-1))2^{-l-1}}
        =
        \sum_{l=L+1}^\infty 2^{-l}
        =2^{-L}>2\varepsilon.
    \end{align}
    This means that $\Delta_{\set{C}}(\coeff_{k_1})\neq \Delta_{\set{C}}(\coeff_{k_2})$ for $k_1,k_2\in[2^d],\,k_1\neq k_2$ and $\Delta_{\set{C}}(\coeffC)\neq\Delta_{\set{C}}(\coeff_{k})$ for any $k\in[2^d]$.
    Consequently 
    $
    \card{\mathcal{C}}
    \geq
    \card{\operatorname{range}(\Delta_\set{C}|_{\RadHaarSpace})}
    \geq 2^d+1>2^d$.
\end{proof}
Based on the lower bound we next show that truncation of the series is indeed the optimal discretization map among all possible discretization maps in $L^\infty([0,1))$ for a specific sequence of bits.
We show this optimality by showing that truncation achieves the lower bound from before.
\begin{theorem}[Truncation Optimality]
\label{thm:optimal_codec}
    Let $L_0\in\nn_0$, $L\in\nn_{-1}$ with $L\geq L_0-1$, $j\in[2^{L_0}]$, and $B=2^{\max\{L-L_0,0\}+1}+\min\{L,L_0\}-1$.
    Then the $L$-term truncation $\Delta_{\RadHaarSpace[L]}$ is an optimal $B$-bit discretization map for the class of Rademacher-Haar wavelet series w.r.t. $L^\infty(\dyadicU{L_0}{j})$ in the sense that it achieves the lower bound in \cref{thm:discretization_map_bound}, that is,
    \begin{align}
        \sup_{g\in\RadHaarSpace}\norm{g-\Delta_{\RadHaarSpace[L]}}[{L^\infty(\dyadicU{L_0}{j})}]=2^{-\max\{L,L_0-1\}-1}.
    \end{align}
\end{theorem}
\begin{proof}
    As in \cref{thm:discretization_map_bound} we restrict to $\dyadicU{L_0}{j}$ with $j=0$ and only consider the reduced coefficient sets $\CoeffSpace_{L_0}$ and $\CoeffSpace[L]_{L_0}$
    For an arbitrary $\coeff\in\CoeffSpace_{L_0}$ 
    it follows immediately from \cref{thm:truncation_error} for $L\geq L_0-1$ that
    \begin{align}
        \norm{g(\cdot;\coeff)-g_L(\cdot;\coeff)}[{L^\infty(\dyadicU{L_0}{0})}]=
        \max_{i\in[2^{L-L_0+1}]}\norm{g(\cdot;\coeff)-g_L(\cdot;\coeff)}[{L^\infty(\dyadicU{L+1}{i})}]
        =
        2^{-L-1}.
    \end{align}
\end{proof}

Note that this optimality only holds if the number of bits is a certain power of two.
In the next result we extend this to arbitrary $B\in \nn$ and also develop a purely polynomial rate bound.

\begin{theorem}[General Rate Bounds]
    Let $L_0\in\nn_0$ and $j\in[2^{L_0}]$.
    Then, for any
    $B\geq L_0$ 
    \begin{align}
        \frac{2^{-L_0-1}}{B-L_0+1}
        \leq
        \inf_{\substack{\set{C}\subset L^\infty(\dyadicU{L_0}{j})\\\card{C}=2^B}}
        \sup_{g\in\RadHaarSpace}
        \norm{g-\Delta_\set{C}(g)}[L^\infty(\dyadicU{L_0}{j})]
        \leq
        \frac{2^{-L_0+1}}{B-L_0+1},
    \end{align}
    and furthermore for any $B\geq \max\{1,L_0\}$
    \begin{align}
        \frac{2^{-L_0-1}}{\max\{2-L_0,1\}}\frac{1}{B}
        \leq
        \inf_{\substack{\set{C}\subset L^\infty(\dyadicU{L_0}{j})\\\card{C}=2^B}}
        \sup_{g\in\RadHaarSpace}
        \norm{g-\Delta_\set{C}(g)}[L^\infty(\dyadicU{L_0}{j})]
        \leq
        \frac{\max\{1,L_0\}2^{-L_0+1}}{B}.
    \end{align}
\end{theorem}
\begin{proof}
    Let $L\in\nn_0$, $L\geq L_0-1$, $j\in[2^{L_0}]$ arbitrary and set $\set{X}\defeq L^\infty(\dyadicU{L_0}{j})$.
    From \cref{thm:optimal_codec} we know the precise error for $L$-level truncation as a $B$-bit discretization map with $B=2^{\max\{L-L_0,0\}+1}+\min\{L,L_0\}-1$ bits.
    We first unify the notation:
    For $L\geq L_0$
    \begin{align}
        B
        &=
        2^{L-L_0+1}+L_0-1
    \intertext{and for $L=L_0-1$}
        B&=2^1+L_0-1-1=L_0=2^{L-L_0+1}+L_0-1.
    \end{align}
    Therefore, for $L\geq L_0-1$ we can use the same expression for $B$ as a function of $L$ and also get the inverse $L=\ld(B-L_0+1)+L_0-1$.
    For every $B\in\nn{}$ such that $L\in\nn_{-1}$ with $L\geq L_0-1$ we can then insert this into the expression for the error of \cref{thm:optimal_codec} and get
    \begin{align}
        \sup_{g\in\RadHaarSpace}\norm{g-\Delta_{\RadHaarSpace[L]}}[\set{X}]=2^{-L-1}=2^{-\ld(B-L_0+1)-L_0}=\frac{2^{-L_0}}{B-L_0+1}.
    \end{align}

    The smallest possible $B$ to fulfill the above criteria is $B=L_0$.
    For general $B\geq L_0$, we know that the approximation error is a positive and monotonically decreasing sequence.
    For $B$ such that the corresponding $L$ is non-integer, we can therefore upper bound the error of the unknown optimal discretization map $\Delta_\mathcal{C}$ by rounding down to the next smaller integer $L$ and taking the error of the $L$-level truncation as upper bound.
    For general $B\in\nn$ with $B\geq L_0$, this is then
    \begin{align}
        \underline{L}&\defeq \lfloor \ld(B-L_0+1)+L_0-1\rfloor\qquad\text{and}\qquad
        \underline{B}\defeq 2^{\underline{L}-L_0+1}+L_0-1.
    \end{align}
    With monotonicity we get
    \begin{align}
        \inf_{\substack{\set{C}\subset \set{X}\\\card{C}=2^B}}
        \sup_{g\in\RadHaarSpace}
        \norm{g-\Delta_\set{C}}[\set{X}]
        &\leq
        \inf_{\substack{\set{C}\subset \set{X}\\\card{C}=2^{\underline{B}}}}
        \sup_{g\in\RadHaarSpace}\norm{g-\Delta_\set{C}}[\set{X}]
        \\&
        =
        \sup_{g\in\RadHaarSpace}\norm{g-\Delta_{\RadHaarSpace[\underline{L}]}}[\set{X}]
        =
        2^{-\underline{L}-1}
        =
        2^{-\lfloor\ld(B-L_0+1)+L_0-1\rfloor-1}
        \\&
        \leq 
        2^{-(\ld(B-L_0+1)+L_0-2)-1}
        =
        \frac{2^{-L_0+1}}{B-L_0+1}.
    \end{align}
    Similarly, with
    \begin{align}
        \overline{L}&\defeq \lceil \ld(B-L_0+1)+L_0-1\rceil,\qquad\text{and}\qquad
        \overline{B}\defeq 2^{\overline{L}-L_0+1}+L_0-1,
    \end{align}
    and monotonicity we get
    \begin{align}
        \inf_{\substack{\set{C}\subset \set{X}\\\card{C}=2^B}}
        \sup_{g\in\RadHaarSpace}
        \norm{g-\Delta_\set{C}}[\set{X}]
        &\geq
        \inf_{\substack{\set{C}\subset \set{X}\\\card{C}=2^{\overline{B}}}}
        \sup_{g\in\RadHaarSpace}\norm{g-\Delta_\set{C}}[\set{X}]
        \\&
        =
        \sup_{g\in\RadHaarSpace}\norm{g-\Delta_{\RadHaarSpace[\overline{L}]}}[\set{X}]
        =
        2^{-\overline{L}-1}
        =
        2^{-\lceil\ld(B-L_0+1)+L_0-1\rceil-1}
        \\&
        \geq 
        2^{-(\ld(B-L_0+1)+L_0)-1}
        =
        \frac{2^{-L_0-1}}{B-L_0+1}.
    \end{align}

    By excluding the possibility of $B=0$, we can further develop (loosen) the bounds by using
    \begin{align}
        \frac{1}{\max\{2-L_0,1\}}\frac{1}{B}
        \leq
        \frac{1}{B-L_0+1}
        \leq
        \frac{\max\{1,L_0\}}{B}
    \end{align}
    for all $B\geq \max\{1,L_0\}$.
\end{proof}

Next, we show that the Sawtooth-Walsh functions form an orthogonal feature map for the Rademacher-Haar wavelet series and that the composition can be equivalently written as a multiplicative structure.

\begin{lemma}
\label[lemma]{thm:walsh_modes}
    Let $L\in \nn_0$, $g\in\RadHaarSpace$, and define $f_i\defeq \saw{L}\cdot\walsh{L}{i}$ for all $i\in[2^L]$.
    Then for $i,j\in[2^L]$ we get
    a simplification of the composition with Walsh functions
    \begin{align}
        \norm{\walsh{L}{i}\circ g-\walsh{L}{i}\circ g_{L-1}}[{L^\infty(\dyadicU{0}{0})}]=0,
    \end{align}
    the equivalent multiplicative representation
    \begin{align}
        \norm{f_i\circ g - \left(2^{L}\sqrt{3}(g-g_{L-1})\right)\cdot(\omega_i\circ g_{L-1})}[{L^\infty(\dyadicU{0}{0})}]=0,
    \end{align}
    and the orthogonality $\innerProd{f_i\circ g}{f_j\circ g}=\delta_{i=j}$.
\end{lemma}

\begin{proof}
We first show that any realization of the Rademacher-Haar wavelet series is essentially a permutation of dyadic sets.    
For every $i\in[2^L]$ we get that $g_{L-1}|_{\dyadicU{L}{i}}$ is constant and overall
\begin{align}
    \operatorname{range}(g_{L-1})
    =
    \left\{\sum_{l=0}^{L-1}\pm 2^{-l-1}\right\}
    =
    \{-1+(2j+1)2^{-L}|j\in[2^L]\}.
\end{align}
As a property of the Rademacher-Haar wavelet series, we get that $g_{L-1}(x)=g_{L-1}(y)$ for $x\in\dyadicU{L}{i}$ and $y\in\dyadicU{L}{j}$ only if $i=j$.
That is, there is a permutation $\sigma_i$ of $[2^L]$ such that for all $i\in[2^L]$ and $x\in\dyadicU{L}{i}$ we have $g_{L-1}(x)=-1+(2\sigma_i+1)2^{-L}$.
\cref{thm:truncation_error} then implies\footnote{Note that \cref{thm:truncation_error} does not tell us whether $g(x)-g_{L-1}(x)$ or $g_{L-1}(x)-g(x)$ actually attain the supremum.
As the effects of $\saw{L}$ and $\walsh{L}{j}$ are limited to the half-open intervals $\dyadicV{L}{j}$, we have to exclude points $x$ for which $g(x)-g_{L-1}(x)$ attains the supremum from our analysis.
As an example, consider the realization with all coefficients being $+1$ and the realization with all coefficients being $-1$.
In the former case, the resulting function is $1-2x$ and in the latter $2x-1$.
On each $\dyadicU{L}{i}$, the former attains only the supremum of $g_{L-1}(x)-g(x)$, whereas the latter attains the supremum of $g(x)-g_{L-1}(x)$.}
\begin{align}
    g(x)
    &\in
    \left[g_{L-1}(x)-2^{-L},g_{L-1}(x)+2^{-L}\right]
    \\&
    =
    \left[-1+2\sigma_i2^{-L},-1+(2\sigma_i+2)2^{-L}\right]
    =
    2^{-L+1}[\sigma_i,\sigma_i+1]-1
    \\&
    =
    \closure{\dyadicV{L}{\sigma_i}}.
\end{align}

Note that $g(x)=\sup \dyadicV{L}{\sigma_i}=g(y)$ only if $x=y$ and therefore $g$ only attains the supremum on a null-set.
As the (stretched) Walsh functions $\walsh{L}{i}$ are constant on every dyadic set $\dyadicV{L}{j}$ this means
\begin{align}
    \norm{\walsh{L}{i}(g(x))-\walsh{L}{i}(g_{L-1}(x))}[{L^\infty(\dyadicU{L}{j})}]=0
\end{align}
for all $i,j\in[2^L]$, which proves the first result.

Now let $j\in[2^L]$.
For any $y\in\dyadicV{L}{j}$ we then see with $M\defeq 2^L$ that $\lfloor M(y+1)/2\rfloor=j$ and therefore for $x\in\dyadicU{L}{i}$ such that $g(x)\in\dyadicV{L}{j}$ we have
\begin{align}
    \saw{L}(g(x))
    &=
    2\sqrt{3}(M(g(x)+1)/2-\lfloor M(g(x)+1)/2\rfloor-1/2)
    \\&
    =
    2\sqrt{3}(Mg(x)/2+M/2-j-1/2)
    \\&
    =
    2^{L}\sqrt{3}(g(x)-(-1+(2j+1)2^{-L}))
    \\&
    =
    2^{L}\sqrt{3}(g(x)-g_{L-1}(x)).
\end{align}
For $x\in\dyadicU{0}{0}$ such that $g(x)\in\dyadicV{L}{j}$ for some $j\in[M]$, the composition $F_i=f_i\circ g$ is then $F_i(x)=(2^{L}\sqrt{3}(g(x)-g_{L-1}(x)))\cdot(\omega_i\circ g_{L-1})$.
Note that the assumption $g(x)\in\dyadicV{L}{j}$ excludes the points $x$ for which $g(x)-g_{L-1}(x)=2^{-L}$ from our analysis, however, as this is only a finite set, it does not influence the essential supremum and thereby also not the approximation bounds in $L^\infty$.

For the orthogonality, we first observe that for $x\in\dyadicU{L}{i}$
\begin{align}
    \walsh{L}{j}(g_{L-1}(x))
    =
    \walsh{L}{j}(2^{-L+1}\sigma_i-1)
    =
    \walshVec{L}{j}[\sigma_i].
\end{align}
We can now use the almost everywhere equivalence from above and the orthogonality of the elements of the wavelet series and get
\begin{align}
    \innerProd{f_i\circ g}{f_j\circ g}
    &=
    \innerProd{2^L\sqrt{3}(g-g_{L-1})\cdot (\walsh{L}{i}\circ g_{L-1})}{2^L\sqrt{3}(g-g_{L-1})\cdot (\walsh{L}{j}\circ g_{L-1})}
    \\&
    =3\cdot 2^{2L}\innerProd{(g-g_{L-1})^2\cdot(\walsh{L}{i}\circ g_{L-1})}{\walsh{L}{j}\circ g_{L-1}}
    \\&
    =
    3\cdot 2^{2L}\sum_{m=0}^{2^L-1}\walshVec{L}{i,\sigma_m}\walshVec{L}{j,\sigma_m}
    \int_{\dyadicU{L}{m}}\left(\sum_{l=L}^\infty\sum_{k=0}^{2^l-1}2^{-l-1}\coeff[l][k]\wavelet(2^lx-k)\right)^2\dee x
    \\&
    =
    3\cdot 2^{2L}\sum_{m=0}^{2^L-1}\walshVec{L}{i,\sigma_m}\walshVec{L}{j,\sigma_m}
    \int_{\dyadicU{L}{m}}\sum_{l=L}^\infty\sum_{k=0}^{2^l-1}\left(2^{-l-1}\right)^2\chi_{\dyadicU{l}{k}}(x)\dee x
    \\&
    =
    3\cdot 2^{2L}\sum_{m=0}^{2^L-1}\walshVec{L}{i,\sigma_m}\walshVec{L}{j,\sigma_m}
    \abs{\dyadicU{L}{m}}2^{-2L-2}\sum_{l=0}^\infty 4^{-l}
    \\&
    =
    \delta_{ij}.
\end{align}
\end{proof}

\subsection{Bounds for Separate and Joint Approximation}
\label{sec:sep_joint_bounds}
\begin{theorem}[Complexity of Separate Approximation]
    \label{thm:wavelet_separated}
    Let $\set{G}$ be the class of Rademacher-Haar wavelets, $M\in\nn$, $n\in\nn[M]$, $L=\lceil\ld(M)\rceil$, $(f_0,\ldots,f_{M-1})$ be the first $M$ Walsh-Sawtooth functions of order $L$, and $\set{F}\defeq \{(f_0\circ g,\ldots,f_{M-1}\circ g)|g\in\set{G}\}$.
    Then, with $N=\min_{i\in[M]}n_i$,
    \begin{align}
        \inf_{(\encoder,\decoder)\in\sepClass[n]}
        \sup_{h\in\set{F}}
        \norm{h-\decoder\circ\encoder(h)}[{L^\infty([0,1];\rr[M])}]
        \geq
        \sqrt{3}
        2^{L-\lceil\ld(N+2^{L})\rceil}
        \geq
        \frac{2^{L-1}\sqrt{3}}{N+2^{L}}.
    \end{align}
\end{theorem}

\begin{proof}
In separated approximation, we discretize each element of the $M$-tuple $F\in\set{F}$ individually.
Instead of considering the class $\sepClass[n]$ of encoder/decoder pairs directly for the full tuple, we first consider element-wise discretization maps $\Delta_{\set{C}}$ with $\set{C}\subset L^\infty([0,1])$ and then combine the element-wise distortion, measured in $L^\infty([0,1])$, into the distortion of the full tuple, measured in $L^\infty([0,1];\rr[M])$.
We then have the equivalent formulation
\begin{align}
    &\kern-2em\inf_{(\encoder,\decoder)\in\sepClass[n]}
    \sup_{g\in\set{G}}
    \norm{g-\decoder\circ\encoder(g)}[{L^\infty([0,1];\rr[M])}]
    \\&
    =
    \inf_{\substack{\set{C}_i\subset L^\infty([0,1]),\\\card{\set{C}}=2^{n_i},\\i\in[M]}}
    \sup_{h\in\set{F}}
    \max_{j\in[M]}
    \norm{h_j-\Delta_{\set{C}_j}(h_j)}[{L^\infty([0,1])}]
    \\&
    =
    \inf_{\substack{\set{C}_i\subset L^\infty([0,1]),\\\card{\set{C}}=2^{n_i},\\i\in[M]}}
    \max_{j\in[M]}
    \sup_{g\in\set{G}}
    \norm{f_j\circ g-\Delta_{\set{C}_j}(f_j\circ g)}[{L^\infty([0,1])}]
    \\&
    =
    \max_{j\in[M]}
    \inf_{\substack{\set{C}_j\subset L^\infty([0,1]),\\\card{\set{C}}=2^{n_j}}}
    \sup_{g\in\set{G}}
    \norm{f_j\circ g-\Delta_{\set{C}_j}(f_j\circ g)}[{L^\infty([0,1])}].
\end{align}
Here, we first simply swap the supremum and the maximum (that is, two suprema) and simplify the expression.
Second, we see that the infimum is over a Cartesian product of families of sets and the argument of the maximum only depends on a single element in the Cartesian product.
Therefore, we can also swap the infimum and the maximum.

We now continue with the analysis for a single index $i\in[M]$.
Let $\coeff\in\CoeffSpace$ be fixed, and choose $\coeffC\in\CoeffSpace$ such that
\begin{alignat}{2}
    \coeffC[l][j]&=\coeff[l][j]\cdot\walsh{L}{i}(g_{L-1}(j2^{-l};\coeff))\qquad&&\text{for}\qquad l\geq L,\,j\in[2^l].
\end{alignat}
Then for all $x\in [0,1)$ we get
\begin{align}
    (2^{L}\sqrt{3})^{-1}\cdot f_i\circ g(x;\coeff)&=(g(x;\coeff)-g_{L-1}(x;\coeff))(\omega_i\circ g_{L-1}(x;\coeff))
    \\&
    =
    \sum_{l=L}^\infty\sum_{j=0}^{2^l-1}\coeff[l][j]\wavelet[l][j](x) \walsh{L}{i}(g_{L-1}(x;\coeff))
    \\&
    =
    \sum_{l=L}^\infty\sum_{j=0}^{2^l-1}\left(\coeff[l][j] \walsh{L}{i}(g_{L-1}(j2^{-l};\coeff))\right)\wavelet[l][j](x)
    \\&
    =
    \sum_{l=L}^\infty\sum_{j=0}^{2^l-1}\coeffC[l][j]\wavelet[l][j](x)
    \\&
    =
    g(x;\coeffC)-g_{L-1}(x;\coeffC).
\end{align}
Here we used that the wavelets $\wavelet[l][j]$ are only supported on $\dyadicU{l}{j}$ and $g_{L-1}$ is piecewise constant such that $g_{L-1}(x;\coeff)=g_{L-1}(j2^{-l};\coeff)$ holds for all $x\in\dyadicU{l}{j}$.
As such an equivalent construction exists for every element of $\CoeffSpace$, the discretization with $B\in\nn$ bits reduces to
\begin{align}
    &\kern-2em\inf_{\substack{\set{C}\subset L^\infty(\dyadicU{0}{0})\\\card{\set{C}}=2^B}}\sup_{g\in\RadHaarSpace}\norm{f_i\circ g-\Delta_{\set{C}}(f_i\circ g)}[{L^\infty(\dyadicU{0}{0})}]
    \\&
    =
    2^L\sqrt{3}\inf_{\substack{\set{C}\subset L^\infty(\dyadicU{0}{0})\\\card{\set{C}}=2^B}}\sup_{\coeffC\in\CoeffSpace}\norm{g(\cdot;\coeffC)-g_{L-1}(\cdot;\coeffC)-\Delta_{\set{C}}(g(\cdot;\coeffC)-g_{L-1}(\cdot;\coeffC))}[{L^\infty(\dyadicU{0}{0})}].
\end{align}

We now show for an arbitrary $L_1\geq L$ that if $\set{C}$ forms an $\varepsilon$-cover of $\{g(\cdot;\coeff)-g_{L-1}(\cdot;\coeff)|\coeff\in\CoeffSpace\}$ with $\varepsilon<2^{-L_1-1}$ then $\ld(\card{C})> 2^{L_1+2}-2^{L+1}$.

Let
\begin{align}
    \CoeffSpace[L:L_1]
    &\defeq
    \{\coeff\in\CoeffSpace[L_1]|\forall l<L:\coeff[l]=1\},
    \\
    N
    &\defeq
    \card{\CoeffSpace[L:L_1]}=2^{\sum_{l=L}^{L_1}2^l}=2^{2^{L_1+1}-2^L},
\end{align}
and consider an enumeration
\begin{align}
    \CoeffSpace[L:L_1]=\{\coeff*_k|k\in[N]\}.
\end{align}
We then construct $\coeff_k,\coeffC_k\in\CoeffSpace$ for $k\in[N]$ such that
\begin{alignat}{2}
    \coeff[l][k,j]&=\coeffC[l][k,j]=\coeff*[l][k,j]&\qquad&\text{for}\qquad l\leq L_1,\, j\in[2^l],\\
    \coeff[l][k,j]&=-\coeffC[l][k,j]=1&&\text{for}\qquad l>L_1,\,j\in[2^l].
\end{alignat}
We see for $k_1,k_2\in[N]$ with $k_1\neq k_2$ that there exist $x\in\dyadicU{0}{0}$ and $L_2\in[L_1+1]$ with $g_{L_2}(x;\coeff_{k_1})=g_{L_2}(x;\coeffC_{k_1})\neq g_{L_2}(x;\coeff_{k_2})=g_{L_2}(x;\coeffC_{k_2})$.
From \cref{thm:wavelet_distance_lower} we then get
\begin{align}
    \norm{g(\cdot;\coeff_{k_1})-g(\cdot;\coeff_{k_2})}[L^\infty(\dyadicU{0}{0})]
    &\geq 2^{-L_2}\geq 2^{-L_1}
    \\
    \norm{g(\cdot;\coeff_{k_1})-g(\cdot;\coeffC_{k_2})}[L^\infty(\dyadicU{0}{0})]
    &\geq 2^{-L_2}\geq 2^{-L_1}
    \\
    \norm{g(\cdot;\coeffC_{k_1})-g(\cdot;\coeff_{k_2})}[L^\infty(\dyadicU{0}{0})]
    &\geq 2^{-L_2}\geq 2^{-L_1}
    \\
    \norm{g(\cdot;\coeffC_{k_1})-g(\cdot;\coeffC_{k_2})}[L^\infty(\dyadicU{0}{0})]
    &\geq 2^{-L_2}\geq 2^{-L_1}.
\end{align}
Furthermore, for any $k\in[N]$ we get
\begin{align}
    g(0;\coeff_k)-g(0;\coeffC_k)
    =
    \sum_{l=L_1+1}^\infty 2^{-l-1}(1-(-1))
    =
    \sum_{l=L_1+1}^\infty 2^{-l}
    =2^{-L_1}.
\end{align}
For an $\varepsilon$-cover with $\varepsilon<2^{-L_1-1}$ and center points $\set{C}$ we get that no pair of elements in
\begin{align}
    \{\coeff_k,\coeffC_k|k\in[N]\}
\end{align}
can be in the same ball of radius $\varepsilon$ and therefore $\card{\set{C}}\geq 2N>N=2^{2^{L_1+1}-2^{L}}$.
From the reverse implication we see for an arbitrary $B\in\nn$ with $B\leq 2^{L_1+1}-2^{L}$ for some $L_1\geq L$ that 
\begin{align}
    \inf_{\substack{\set{C}\subset L^\infty(\dyadicU{0}{0})\\\card{\set{C}}=2^B}}\sup_{g\in\RadHaarSpace}\norm{f_i\circ g-\Delta_{\set{C}}(f_i\circ g)}[{L^\infty(\dyadicU{0}{0})}]\geq \sqrt{3}2^{L-L_1-1}.
\end{align}
Specifically, for any $B\in\nn$ we can choose $\overline{L}_1\defeq\lceil\ld(B+2^{L})\rceil-1\leq \ld(B+2^{L})$ and get $B\leq 2^{\overline{L}_1+1}-2^{L}$, thus,
\begin{align}
    &\kern-2em\inf_{\substack{\set{C}\subset L^\infty(\dyadicU{0}{0})\\\card{\set{C}}=2^B}}\sup_{g\in\RadHaarSpace}\norm{f_i\circ g-\Delta_{\set{C}}(f_i\circ g)}[{L^\infty(\dyadicU{0}{0})}]
    \\&
    \geq
    \sqrt{3}
    2^{L-\overline{L}_1-1}
    \geq
    \sqrt{3}2^{L-\ld(B+2^{L})-1}
    \geq
    \sqrt{3}\frac{2^{L-1}}{B+2^{L}}
\end{align}

Overall, with $N=\min_{i\in[M]}n_i$, 
we get
\begin{align}
    &\kern-2em\inf_{(\encoder,\decoder)\in\sepClass[n]}
    \sup_{h\in\set{F}}
    \norm{h-\decoder\circ\encoder(h)}[{L^\infty([0,1];\rr[M])}]
    \\&
    =
    \max_{j\in[M]}\inf_{\substack{\set{C}\subset L^\infty([0,1])\\\card{\set{C}}=2^{n_j}}}\sup_{g\in\RadHaarSpace}\norm{f_j\circ g-\Delta_{\set{C}}(f_j\circ g)}[{L^\infty([0,1])}]
    \\&
    \geq
    \sqrt{3}2^{L-\lceil\ld(N+2^{L})\rceil}
    \geq 
    \frac{2^{L-1}\sqrt{3}}{N+2^{L}}.
\end{align}
\end{proof}

\begin{theorem}[Complexity of Joint Approximation]
    \label{thm:wavelet_joint}
    Let $\set{G}$ be the class of Rademacher-Haar wavelets, $M\in\nn$, $L=\lceil\ld(M)\rceil$, $N\in\nn$ with $N\geq 2^L$, $(f_0,\ldots,f_{M-1})$ be the first $M$ Walsh-Sawtooth functions of order $L$, and $\set{F}\defeq \{(f_0\circ g,\ldots,f_{M-1}\circ g)|g\in\set{G}\}$.
    Then,
    \begin{align}
        \inf_{(\encoder,\decoder)\in\jointClass[N]}
        \sup_{h\in\set{F}}
        \norm{h-\decoder\circ\encoder(h)}[{L^\infty([0,1];\rr[M])}]
        \leq\frac{2^{L+1}\sqrt{3}}{N+1}.
    \end{align}
\end{theorem}

\begin{proof}
For the joint approach, we observe from the decomposition in \cref{thm:walsh_modes} that for an arbitrary realization $g\in\RadHaarSpace$ it is sufficient to separately approximate $g-g_{L-1}$ using $B_g\in\nn{}$ bits and all compositions $\walsh{L}{i}(g_{L-1})$ using a total of $B_f\in\nn{}$ bits.

In \cref{thm:walsh_modes} we observed that $\walsh{L}{i}(g)\equiv \walsh{L}{i}(g_{L-1})$.
As the functions $\walsh{L}{i}$, $i\in[M]$ are fixed and uniquely identified by the index $i$, it is therefore sufficient, to discretize $g_{L-1}$ in order to get all $\walsh{L}{i}(g_{L-1})$ for all $i\in[M]$.
To do so, a total number of
\begin{align}
    B_f=\sum_{l=0}^{L-1}2^l=2^L-1
\end{align}
bits is sufficient.

For the remainder we again approximate $g-g_{L-1}$ by $g_{L_1}-g_{L-1}$ and take an upper bound by rounding each $B_g\in\nn$ down to the next largest full level $\underline{L}_1\defeq \lfloor \ld(B_g+2^L)-1\rfloor$.
This leads to the error bound
\begin{align}
    &\kern-2em\inf_{\substack{\set{C}\subset L^\infty(\dyadicU{0}{0})\\\card{\set{C}}=2^{B_g+B_f}}}\sup_{g\in\RadHaarSpace}\norm{f_i\circ g-\Delta_{\set{C}}(f_i\circ g)}[{L^\infty(\dyadicU{0}{0})}]
    \\&
    \leq
    2^L\sqrt{3}\inf_{\substack{\set{C}\subset L^\infty(\dyadicU{0}{0})\\\card{\set{C}}=2^{B_g}}}\sup_{g\in\RadHaarSpace}\norm{(g-g_{L-1}-\Delta_{\set{C}}(g-g_{L-1}))\cdot\omega_i(g)}[{L^\infty([0,1])}]
    \\&
    \leq
    2^L\sqrt{3}\sup_{g\in\RadHaarSpace}\norm{g-g_{L-1}-g_{\underline{L}_1}-g_{L-1}}[{L^\infty(\dyadicU{0}{0})}]
    \\&
    =
    2^L\sqrt{3}\sup_{g\in\RadHaarSpace}\norm{g-g_{\underline{L}_1}}[{L^\infty(\dyadicU{0}{0})}]
    \\&
    =
    2^L\sqrt{3}2^{-\underline{L}_1-1}
    \leq
    2^L\sqrt{3}2^{-\ld(B_g+2^L)+1}
    =
    \frac{2^{L+1}\sqrt{3}}{B_g+2^L}.
\end{align}
For $N\geq 2^L$ we can allocate $B_f=2^L-1$ and $B_g=N-2^L+1$.
We then get perfect discretization of $\walsh{L}{i}(g_{L-1})$ for all $i\in[M]$ and
\begin{align}
    \inf_{(\encoder,\decoder)\in\jointClass[N]}
    \sup_{h\in\set{F}}
    \norm{h-\decoder\circ\encoder(h)}[{L^\infty([0,1];\rr[M])}]
    \leq
    \frac{2^{L+1}\sqrt{3}}{B_g+2^L}
    \leq
    \frac{2^{L+1}\sqrt{3}}{N+1}.
\end{align}
\end{proof}

\subsection{Proof of \texorpdfstring{\cref{thm:approximation_gap}}{Main Theorem}}
\label{sec:proof_main_theorem}

We can now finally combine the results from above in order to prove \cref{thm:approximation_gap}.

\begin{proof}
    Let $n\in\nn[M]$ with $n_i\geq 2$ and $N=M\min_{i\in[M]}n_i$. With $n_i\geq 2$, we get $N\geq 2M=2^{\ld(M)+1}\geq2^{\lceil \ld(M)\rceil}=2^L$. Therefore, $N$ is a valid number of bits for \cref{thm:wavelet_joint} and $N/M=\min_{i\in[M]}n_i$ is valid for \cref{thm:wavelet_separated}.
    
   Due to the optimality of truncation codes for the Rademacher-Haar wavelet series of \cref{thm:optimal_codec}, we know that the infima in \cref{thm:wavelet_separated} and \cref{thm:wavelet_joint} are indeed attained.
    We denote $\set{C}_J^N$ as the optimal joint discretization set and $\set{C}_S^n$ as the optimal separate discretization set.

    The upper bound on the supremum for joint discretization in \cref{thm:wavelet_joint} then immediately translates to a for-all statement over $\set{F}$.
    For separate discretization we first observe in case $N/M=2^K-2^L$ for some $K\in\nn$ that the left side in the lower bound
    \begin{align}
        \sqrt{3}2^{L-\lceil\ld(N/M+2^L)\rceil}
        =
        \sqrt{3}2^{L-K}
        \geq
        \frac{2^{L-1}\sqrt{3}}{N/M+2^L}
    \end{align}
    is indeed attained by a truncation (compare \cref{thm:optimal_codec}) and that each realization has the same distance to the truncation (see \cref{thm:truncation_error}).
    Therefore, the supremum in \cref{thm:wavelet_separated} is indeed attained for all elements in $\set{F}$ and the lower bound holds for all $h\in\set{F}$.
    Second we observe that if $N/M\neq 2^K-2^L$ for all $K\in\nn$ the inequality
    \begin{align}
        \sqrt{3}
        2^{L-\lceil\ld(N/M+2^{L})\rceil}
        >
        \frac{2^{L-1}\sqrt{3}}{N/M+2^{L}}
    \end{align}
    is strict.
    Therefore, there exists an $h\in\set{F}$ for which the right side is still a valid lower bound.

    Overall, we can combine the upper bound for joint discretization
    \begin{align}
        \norm{h-\Delta_{\set{C}_J^N}(h)}[{L^\infty([0,1];\rr[M])}]
        &\leq
        \frac{2^{L+1}\sqrt{3}}{N+1}
        \intertext{and the lower bound for separate discretization}
        \norm{h-\Delta_{\set{C}_S^{n\vphantom{N}}}(h)}[{L^\infty([0,1];\rr[M])}]
        &\geq
        \frac{2^{L-1}\sqrt{3}}{N/M+2^{L}},
    \end{align}
    which results in
    \begin{align}
        \frac{\norm{h-\Delta_{\set{C}_S^{n\vphantom{N}}}(h)}[{L^\infty([0,1];\rr[M])}]}{\norm{h-\Delta_{\set{C}_J^N}(h)}[{L^\infty([0,1];\rr[M])}]}
        \geq
        \frac{M}{4}\frac{1+\frac{1}{N}}{1+\frac{2^{L}M}{N}}
        \geq
        \frac{M}{4}\frac{1+\frac{1}{N}}{1+\frac{2M^2}{N}}.
    \end{align}
\end{proof}

\section{Approximation by Heaviside Neural Networks}
\label{sec:heaviside_approx}
We now realize the codec separation in a neural network model. The choice of activation is dictated by the function class itself. Generic members of $\mathcal{G}$ have jump discontinuities of constant height at dyadic points, and the uniform distance from a function with jump height $2h$ to any continuous function is at least $h$. Consequently no network with a continuous activation, in particular no ReLU network, can uniformly approximate $\mathcal{F}$ below a constant error, independently of width
and depth.

For this we construct a neural network with Heaviside activation function that indeed approximates $f_i\circ g$ efficiently.
From the equivalent representation in \cref{thm:walsh_modes}, we 
see that $f_i\circ g$ decomposes into a product of two functions that can be discretized to piecewise constant functions, which then can trivially be represented by Heaviside neural networks.
On the one hand, we have the function $\walsh{L}{i}\circ g_{L-1}$ which has $2^L\leq 2M$ constant pieces and can therefore be represented by a Heaviside neural network with a single hidden layer of width $2^L$.
On the other hand, we have the tail of the wavelet series $g-g_{L-1}$.
For this function, we can look at the truncation $g-g_{L-1}$ to a level $L_1\geq L$ as in \cref{thm:wavelet_joint} which results in the given bound.
We see that these truncations are again piecewise constant with $2^{L_1+1}$ pieces.
In order to realize a joint approximation, we now require a neural network that can represent or approximate a pointwise multiplication.

A well-known approximation of pointwise multiplication by neural networks is given by \cite{Yarotsky17ErrorBoundsApproximations}, where a composition of triangle functions is used in order to approximate a square function.
Based on the square function, it is then possible to approximate multiplication via $2a\cdot b=(a+b)^2-a^2-b^2$.
With the construction of \cite{Yarotsky17ErrorBoundsApproximations}, this approximation can be carried out in logarithmic complexity in terms of the error $\varepsilon$.
That is, if we can approximate a triangle function or equivalently the absolute value on the domain $[-1,1]$ in logarithmic complexity, which implies exponential convergence, then the overall approximation of multiplication is sufficiently efficient so that the dominating term in approximating $(g-g_{L-1})\cdot (\walsh{L}{i}\circ g_{L-1})$ is the approximation of $g-g_{L-1}$.

In our approach, we restrict ourselves to combinations of affine transformations, a single type of nonlinear function, namely the Heaviside function, and skip connections.
Note that under these restrictions we cannot simply consider $x\cdot\chi_{[0,\infty)}(x)-x\cdot\chi_{(-\infty,0)}(x)$.
In order to achieve exponential convergence towards a piecewise linear function, we require a construction where the number of constant pieces increases exponentially with the number of parameters.
For neural networks, this can be achieved by keeping a constant width throughout the network and just letting the depth grow.
As the absolute value is an even function, we aim to build an ever finer and symmetric partition of the domain and the corresponding characteristic functions.
As a first step, we consider the function
\begin{align}
    \sigma_1\deffun[-1,1]\to\rr,\qquad\sigma_1(x)\defeq-\chi_{(-\infty,-0.5)}(x)+\chi_{[0.5,\infty)}(x).
\end{align}
We then recursively build odd sawtooth functions for which the number of segments increases exponentially in the number of levels $K\in\nn$ via the recursion
\begin{alignat}{3}
    \varphi_0&\deffun[-1,1]\to\rr,\qquad&\varphi_0(x)&\defeq x,&&\\
    \varphi_i&\deffun[-1,1]\to\rr,&\varphi_i(x)&\defeq 2\varphi_{i-1}(x)-\sigma_1(\varphi_{i-1}(x)),&\quad&\text{for}\,1\leq i\leq L.
\end{alignat}
The recursive construction of sawtooth functions with exponentially many segments follows the mechanism underlying the depth-separation results of \cite{Telgarsky16BenefitsDepthNeural}.
The resulting functions are shown in \cref{fig:abs_approx_functions} (left) for $i\in\{0,1,2\}$.
From this we can then build characteristic functions of symmetric dyadic sets via
\begin{alignat}{3}
    \sigma_2&\deffun[-1,1]\to\rr,\qquad&\sigma_2(x)&\defeq\chi_{(-\infty,-0.5)}(x)+\chi_{[0.5,\infty)}(x)
    \intertext{and}
    \psi_i&\deffun[-1,1]\to\rr,&\psi_i(x)&\defeq\sigma_2(\varphi_i(x)),
\end{alignat}
which results in the functions shown in \cref{fig:abs_approx_functions} (middle).
Multiplication with appropriate powers of two and summation then finally leads to the approximation
\begin{alignat}{3}
    \Psi_{K}&\deffun[-1,1]\to\rr,\qquad&\Psi_{K}(x)&=\sum_{l=0}^{K-1}2^{-l-1}\psi_l(x)%\sigma_2(\varphi_l(x))
\end{alignat}
of the absolute value by a step function, which is shown in \cref{fig:abs_approx_functions} (right).
\begin{figure}
    \centering
    \includegraphics[page=1]{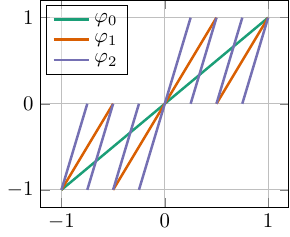}
    \hfill
    \includegraphics[page=2]{figures/heaviside_network/functions.pdf}
    \hfill
    \includegraphics[page=3]{figures/heaviside_network/functions.pdf}
    \caption{Intermediate steps of approximating the absolute value by neural networks with Heaviside activation function.
    The input is recursively turned into an odd sawtooth function (left), which is then transformed into characteristic functions of certain dyadic sets (middle), and finally summed up to an approximation of the absolute value (right).}
    \label{fig:abs_approx_functions}
\end{figure}
The corresponding feedforward network structure with skip connections is depicted in \cref{fig:abs_approx}.
\begin{figure}
    \centering
    \includegraphics{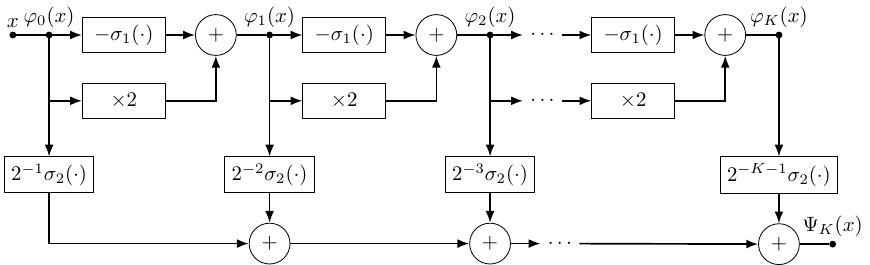}
    \caption{Recursive network structure for approximating the absolute value by neural networks with Heaviside activation function.
    Each layer consists of four neurons with Heaviside activation function (two for $\sigma_1$ and two for $\sigma_2$) and two skip connections (one to the next layer and one to the last layer).}
    \label{fig:abs_approx}
\end{figure}
Overall the resulting function has $2^{K+1}$ constant pieces with equal length, which results in the function
\begin{align}
    \Psi_K(x)=
    \begin{cases}
        2^{-K}\lfloor2^{K}\abs{x}\rfloor,&\abs{x}<1,\\
        1-2^{-K},&\abs{x}=1,        
    \end{cases}
\end{align}
and reaches a uniform error of $\varepsilon\in \mathcal{O}(2^{-K})$.

\begin{lemma}[Approximation of Pointwise Product]
\label[lemma]{thm:product}
    Let $S\in\nn$. There exists a Heaviside neural network $\hat{\times}_S:[0,1]^2\to[0,1]$ of width $12$ and depth $O(S^2)$, that achieves for any $a,b\in[0,1]$
    \begin{align}
        \abs{a\cdot b-\hat{\times}_S(a,b)}\lesssim S2^{-S}.
    \end{align}
\end{lemma}
\begin{proof}
The proof of our statement is based on the composition of triangle functions of \cite{Yarotsky17ErrorBoundsApproximations}.
In order to get a final error estimate for the square function, we first have to track, how the error of our approximation of the absolute value propagates through the composition.

Let $\delta:[0,1]\to[0,1]$ with $\delta(x)\defeq 1-2\abs{x-1/2}$, $\hat{\delta}:[0,1]\to[0,1]$ with $\hat{\delta}(x)\defeq 1-2\Psi_K(x-1/2)$, and for $T \geq 2$ define
\begin{alignat}{2}
    \delta_1&\defeq \delta,&\hat{\delta}_1&\defeq \hat{\delta}\\
    \delta_T&\defeq \delta\circ\delta_{T-1},&\hat{\delta}_T&\defeq \hat{\delta}\circ\hat{\delta}_{T-1}\\
    \varepsilon_T&\defeq \norm{\delta_T-\hat{\delta}_T}[L^\infty([0,1])].
\end{alignat}
We show that $\varepsilon_T\leq (2^T-1)\varepsilon_1$, for which we use three major steps.
First, an error decomposition with triangle inequality,
second, the Lipschitz-continuity of $\delta$ with Lipschitz constant $2$,
and third, dropping the first layer in a compositional structure bounds the $L^\infty$-norm if the image of the first layer is a subset of the domain of the norm.
That is, for an arbitrary $T$ we get
\begin{align}
    \norm{\delta_T-\hat{\delta}_T}[L^\infty([0,1])]
    &=
    \norm{\delta_{T}-\delta\circ\hat{\delta}_{T-1}+\delta\circ\hat{\delta}_{T-1}-\hat{\delta}_{T}}[L^\infty([0,1])]
    \\&
    \leq
    \norm{\delta_{T}-\delta\circ\hat{\delta}_{T-1}}[L^\infty([0,1])]
    +
    \norm{\delta\circ\hat{\delta}_{T-1}-\hat{\delta}_{T}}[L^\infty([0,1])]
    \\&
    \leq
    2\norm{\delta_{T-1}-\hat{\delta}_{T-1}}[L^\infty([0,1])]
    +
    \norm{\delta-\hat{\delta}}[L^\infty([0,1])]
    \\&
    =
    2\varepsilon_{T-1}
    +
    \varepsilon_1
    \\&
    \leq
    2(2^{T-1}-1)\varepsilon_1+\varepsilon_1
    =
    (2^T-1)\varepsilon_1.
\end{align}
Furthermore, $\abs{\cdot}-\Psi_K(\cdot)$ attains the maximal error for $x\in\{\pm 1\}$ which leads to
\begin{align}
    \varepsilon_1
    &=
    \norm{\delta-\hat{\delta}}[L^\infty([0,1])]
    \\&
    =
    2\norm{\abs{\cdot}-\Psi_K(\cdot)}[L^\infty([-1,1])]
    \\&
    =
    2(\abs{1}-\Psi_K(1))
    =
    2(1-1+2^{-K-1})
    =
    2^{-K}
\end{align}
and consequently,
\begin{align}
    \varepsilon_T\leq(2^T-1)2^{-K}.
\end{align}
For a fixed $T\in\nn$ we now choose $K=2T$ and can approximate $\delta_T$ with a Heaviside neural network of depth $2T^2$ and width 4 at an error of $\mathcal{O}(2^{-T})$.
The construction for the product structure of \cite{Yarotsky17ErrorBoundsApproximations} is based on linear combinations of $\delta_1,\ldots,\delta_S$ for $S\in\nn$ and results in an error of $\mathcal{O}(2^{-2S-2})$.
For $T=S$, the final error bound is then 
\begin{align}
    \abs{a\cdot b-\hat{\times}_S(a,b)}
    \lesssim 2^{-2S-2}+S2^{-S}
    \lesssim S2^{-S},
\end{align}
where we get the additional $S2^{-S}$ from the error in each $\delta_i,\,i\in[M]$.

Due to the formulation of the product via the formula $2ab=(a+b)^2-a^2-b^2$, we need to repeat the same network three times in width, while the depth stays at $\mathcal{O}(2T^2)=\mathcal{O}(2S^2)=\mathcal{O}(S^2)$.
\end{proof}

\begin{theorem}[Joint Approximation by Heaviside Networks]
\label{thm:heaviside_joint}
    Let $M \in \nn$, $L = \lceil \ld(M) \rceil$, and let $\mathcal{F}$ be the class of orthogonal target functions of \cref{thm:approximation_gap}.
    For $N,S\in\nn$ with $N \geq 2^{L+1}-1$, there exists a composition of a shallow neural network $\Phi^{(I)}$ of width $N+2M^2$ and $M+1$ outputs and a deep neural network $\Phi^{(O)}$ of width $12M$ 
    and depth $\mathcal{O}(S^2)$ such that for each $h\in\set{F}$ only the parameters of $\Phi^{(I)}$ need to be adapted to get
    \begin{align}
        \norm{\Phi^{(O)}\circ \Phi^{(I)} - h}[L^\infty([0,1];\rr[M])]\lesssim M(N^{-1}+S2^{-S}),
    \end{align}
    where the hidden constant is independent of $M$.
\end{theorem}

\begin{proof}
For $L_1\geq L$ the truncation $g_{L_1}$ is a piecewise constant function with $2^{L_1+1}$ pieces.
It can therefore be realized exactly by a shallow Heaviside network $\Phi_N$ of width $N=2^{L_1+1}-1$.
For an arbitrary $N\in\nn$ we now choose $\underline{L}_1=\lfloor\ld(N+1)\rfloor-1$ and by using \cref{thm:truncation_error}, we get
\begin{align}
    \norm{g-g_{L_1}}[L^{\infty}([0,1])]
    =
    \norm{g-\Phi_N}[L^{\infty}([0,1])]
    =
    2^{-\underline{L}_1-1}
    \leq
    2^{-\ld(N+1)+1}
    =\frac{2}{N+1}.
\end{align}
Furthermore, for every $i\in [M]$, $\walsh{L}{i}(g_{L-1})$ is a piecewise constant function with $2^L$ pieces and can therefore be realized exactly by a shallow Heaviside network of width $2^L\leq 2M$.

Performing an error decomposition and using the result from \cref{thm:product} for $S\in\nn$, we get\footnote{Note that the factors in the approximation of the product do not adhere to the domain $[0,1]$ but are in $[-1,1]$.
    As the approximation of the product is realized by a neural network, we can simply modify the affine transformation at the input and the output to adapt to the proper domain.}
    \begin{align}
        &\kern-2em\norm{(g-g_{L-1}) \cdot \walsh{L}{i}(g_{L-1})
        - \hat{\times}_S(\Phi_N-g_{L-1},  \walsh{L}{i}(g_{L-1}))}[L^\infty([0,1])]  \\
        &= \big\|
            (g-g_{L-1}) \cdot \walsh{L}{i}(g_{L-1}) 
            -
            (\Phi_N-g_{L-1}) \cdot \walsh{L}{i}(g_{L-1})
            \\&\phantom{{}={}}
            + 
            (\Phi_N-g_{L-1}) \cdot \walsh{L}{i}(g_{L-1})
            -
            \hat{\times}_S(\Phi_N-g_{L-1},  \walsh{L}{i}(g_{L-1}))
            \big\|_{L^\infty([0,1])} \\
        &\leq \norm{
            (g-\Phi_N) \cdot \walsh{L}{i}(g_{L-1})
            -
            (\Phi_N-g_{L-1}) \cdot \walsh{L}{i}(g_{L-1})
            }[L^\infty([0,1])] 
        \\&\phantom{{}={}}
        +
        \norm{
            (\Phi_N-g_{L-1}) \cdot \walsh{L}{i}(g_{L-1})
            -
            \hat{\times}_S(\Phi_N-g_{L-1},  \walsh{L}{i}(g_{L-1}))}[L^\infty([0,1])] \\
        &\lesssim N^{-1} + S2^{-S},
    \end{align}
    where the hidden constant is independent of $M$ and $i$, however, as this approximates $h_i/(\sqrt{3}2^L)$, the final error in the approximation of $h_i$ is bounded by $M(N^{-1}+S2^{-S})$.
    
    Finally, we require one shallow network with $N$ neurons to realize $\Phi_N$, one shallow network with $2^L\leq 2M$ neurons for the realization of each $\walsh{L}{i}(g_{L-1})$, and $M$ realizations of the product network of width $12$ and depth $\mathcal{O}(S^2)$
    in order to achieve an overall error of $\mathcal{O}(M(N^{-1}+S2^{-S}))$.
\end{proof}

\begin{corollary}[Parametric Separation for Quantized Networks]
\label[corollary]{thm:network_gap}
Let $M, L, \mathcal{F}$ be as in \cref{thm:heaviside_joint} and let $N \in \nn$. For $i \in [M]$, let $\Phi_i$ be any family of neural networks of arbitrary architecture and activation function, with $W$ weights each quantized to $w$ bits, where $wW\le N/M.$ Then 
\begin{align}
    \sup_{h \in \mathcal{F}} \max_{i \in [M]}
    \min_{\phi_i\in\Phi_i}
    \norm{h_i - \phi_i}_{L^\infty([0,1])}
     \geq \frac{2^{L-1}\sqrt{3}}{N/M+2^L}
     \geq \frac{M\sqrt{3}/2}{N/M+2M},
\end{align}
whereas the joint architecture of \cref{thm:heaviside_joint}, quantizable in at most $N$ bits, achieves
\begin{align}
    \sup_{h \in \mathcal{F}}
    \norm{h - \Phi(h)}_{L^\infty([0,1];\rr[M])}
    \lesssim
    \frac{M}{N}.
\end{align}
In particular, the multitask gap of \cref{thm:approximation_gap} is realized by neural networks of matched description length.
\end{corollary}
\begin{proof}
    A fixed architecture with $W$ weights, each quantized to $w$ bits, realizes at most $2^{wW}$ distinct functions and is therefore a $B$-bit discretization map with $B = wW \le N/M$ in the sense of \cref{def:disc_maps}.
    The lower bound is then given by \cref{thm:wavelet_separated}.
    The upper bound is \cref{thm:heaviside_joint} with $S=2\lceil\ld(N)\rceil$.
    Furthermore, we observe on the one hand, that the output network $\Phi^{(O)}$ has fixed weights, therefore no dependence on the function $h$.
    On the other hand, the input network encodes the truncation $g_{L_1}$ for $L_1=\lfloor \ld(N+1)\rfloor-1$, thus, the weights are quantizable to $2^{L_1+1}-1\leq N$ bits.
\end{proof}
For real weights the counting argument breaks down entirely, as a single real parameter already carries infinitely many bits, so we restrict to encodable networks, as is standard in the rate-distortion analysis of neural networks \citep{Bolcskei19OptimalApproximationSparsely, elbrachter2021deep}.

%%%%%%%%%%%%%%%%%%%%%%%%%%%%%%%%%%%%%%%%%%%%%%%%%%
%%%%%%%%%%%%%%%%%%%%%%%%%%%%%%%%%%%%%%%%%%%%%%%%%%
%%%%%%%%%%%%%%%%%%%%%%%%%%%%%%%%%%%%%%%%%%%%%%%%%%
\section{Conclusion}

We have shown that joint approximation retains its parametric advantage over separate approximation even for orthogonality-constrained multi-head models. For the compositional class
\[
F_g=(f_0\circ g,\ldots,f_{M-1}\circ g),
\]
the optimal joint codec, using a no larger total bit budget, improves over the best separate codec by an asymptotic factor of $M/4$ in the uniform norm. This gap persists under orthogonality of the coordinates and on the level of finite-bit approximation rates, leaving shared latent approximation complexity as the main source of the gain.

Our construction makes this source explicit. The coefficients of the Rademacher-Haar feature $g$ form a dyadic tree whose unresolved tail fixes the approximation rate, and the truncation estimates give matching finite-bit bounds with explicit constants. The Sawtooth-Walsh maps turn this feature into $M$ orthogonal readouts, so that orthogonality is imposed on the outputs while the hard information remains shared. Thus, each marginal problem contains the same unresolved tail, whereas the vector-valued problem contains it only once. The resulting $M/4$ gap is thus a quantitative rate separation, obtained through a fully constructive comparison of the optimal finite-bit rates. In this sense, our main result gives an approximation-theoretic version of the principle behind shared representations, namely that the dominant cost lies in the latent feature rather than the heads.

The Heaviside construction translates the codec separation into a concrete neural network realization. The relevant truncations of $g$ and the Walsh factors are piecewise constant and, therefore, are naturally represented by the Heaviside subnetworks. Through the triangle-wave construction, the network realizes the product structure. The architecture realizes the same bit-economy as the codec, with a shared trunk for the Rademacher-Haar feature and heads that impose the Sawtooth-Walsh readouts.

Viewed from the perspective of foundation models, the construction supports a simple design principle. Structural requirements may be imposed in the heads without forcing the latent representation to be learned or encoded separately for each output. This points to a more central question: which structural constraints can be carried by the readouts, and which ones force additional complexity into the shared representation? The open problem is to identify when this mechanism persists for broader classes of neural networks, including smoother features, operator-valued maps, learned latent representations, and constraints arising from equivariance, conservation laws or boundary conditions. Within any such setting, one would want to identify the function classes that play the role of the Rademacher-Haar wavelet series, whose coefficients carry equal information at every scale, for those we expect the same factor-$M$ separation.

\bibliography{references}
\end{document}